\pdfoutput=1

\documentclass{article}

\usepackage[utf8]{inputenc} 
\usepackage[T1]{fontenc}  
\usepackage{hyperref}
\usepackage{url} 
\usepackage{booktabs} 
\usepackage{amsfonts} 
\usepackage{nicefrac}  
\usepackage{microtype} 

\usepackage{amsmath}
\usepackage{amsthm}
\usepackage{amssymb}

\usepackage{xcolor} 
\newtheorem{theorem}{Theorem}
\newtheorem{lemma}{Lemma}
\newtheorem{corollary}{Corollary}

\usepackage{booktabs}
\usepackage[toc,page]{appendix}
\usepackage{comment}

\usepackage{graphicx}

\usepackage[inline]{enumitem}

\renewcommand{\[}{\begin{eqnarray}}
\renewcommand{\]}{\end{eqnarray}}

\newcommand{\E}{\mathbb{E}}
\renewcommand{\P}{\mathbb{P}}
\usepackage{dsfont}
\usepackage{subcaption}

\usepackage[authoryear]{natbib}

\usepackage{microtype}
\usepackage{graphicx}
\usepackage{booktabs} 

\usepackage{hyperref}

\usepackage[accepted]{icml2021}

\icmltitlerunning{Learning Randomly Perturbed Structured Predictors for Direct Loss Minimization}

\begin{document}

\twocolumn[
\icmltitle{Learning Randomly Perturbed Structured Predictors \\ for Direct Loss Minimization}



\icmlsetsymbol{equal}{*}

\begin{icmlauthorlist}
\icmlauthor{Hedda Cohen Indelman}{Technion}
\icmlauthor{Tamir Hazan}{Technion}

\end{icmlauthorlist}

\icmlaffiliation{Technion}{Technion}

\icmlcorrespondingauthor{Hedda Cohen Indelman}{cohen.hedda@campus.technion.ac.il}

\icmlkeywords{Machine Learning, ICML} \vskip 0.3in]


\printAffiliationsAndNotice{} 

\begin{abstract}
  Direct loss minimization is a popular approach for learning predictors over structured label spaces. This approach is computationally appealing as it replaces integration with optimization and allows to propagate gradients in a deep net using loss-perturbed prediction. Recently, this technique was extended to generative models, by introducing a randomized predictor that samples a structure from a randomly perturbed score function. In this work, we interpolate between these techniques by learning the variance of randomized structured predictors as well as their mean, in order to balance between the learned score function and the randomized noise. We demonstrate empirically the effectiveness of learning this balance in structured discrete spaces.
\end{abstract}


\section{Introduction}
Learning and inference in high-dimensional structured models drives much of the research in machine learning applications, from computer vision, natural language processing, to computational chemistry. Examples include scene understanding \citep{kendall2015bayesian} machine translation \citep{wiseman2016sequence} and molecular synthesis \citep{jin2020hierarchical}. The learning process optimizes a score for each of the exponentially many structures in order to best fit the mapping between input and output in the training data. While it is often computationally infeasible to evaluate the loss of all exponentially many structures simultaneously through sampling, it is often feasible to predict the highest scoring structure efficiently in many structured settings. 

Direct loss minimization is an appealing approach in discriminative learning that allows to learn a structured model by predicting the highest scoring structure \citep{NIPS2010_4069, Keshet2011DirectER, Song2016TrainingDN}. It allows to improve the loss of the structured predictor by considering the gradients of two predicted structures: over the original loss function and over a perturbed loss function. This approach implicitly uses the data distribution to smooth the loss function between a training structure and a predicted structure, thus propagating gradients through the maximal argument of the predicted structure. Unfortunately, our access to the data distribution is limited and we cannot reliably represent the intricate relation between a training instance and its exponentially many structures. Recently, this framework was extended to generative learning, where a random perturbation that follows the Gumbel distribution law allows to sample from all possible structures \citep{Lorberbom2018DirectOT}. However, one cannot apply this generative learning approach effectively to discriminative learning, since the random noise that is added in the generation process interferes in predicting the best scoring structure.

In this work we combine these two approaches: we explicitly add random perturbation to each of the structures, in order to reliably represent the intricate relation between the a training instance and its exponentially many structures. To balance between the learned score function and the added random perturbation, we treat the score function as the mean of the random perturbation, and learn its variance. This way we are able to control the ratio between the signal (the score) and the noise (the random perturbation) in discriminative learning. 

In summary, we make the following contributions:
\begin{enumerate}[nosep]
    \item We show that the uniqueness assumption of the predicted structure is a key element in the gradient step of direct loss minimization, thus mathematically defining its general position assumption. 
    \item We prove that random perturbation ensures unique maximizers with probability one.
    \item We identify that  random perturbation might also serve as noise that masks the score function signal. Hence, we introduce a method for learning both the mean and the variance of randomized predictors in the high-dimensional structured label setting.
    \item We show empirically the benefit of our approach in two structured prediction problems.
\end{enumerate}


\section{Related Work}
\label{sec:related} 
Effective structured learning and inference over discrete, combinatorial models is challenging and has been addressed by different approaches.

Direct loss minimization is an effective approach in discriminative learning that was devised to optimize non-convex and non-smooth loss functions for linear structured predictors \citep{NIPS2010_4069}. Later it was extended to non-linear models, including hidden Markov models and deep learners \citep{Keshet2011DirectER, Song2016TrainingDN}.  Our work extends direct loss minimization by adding random noise to its structured predictor and learning its variance. Recently, the idea of optimization that replaces sampling was extended to generative learning and reinforcement learning \citep{Lorberbom2018DirectOT, lorberbom2019direct}. Similar to our work, these works also add random Gumbel perturbation and learn the mean of their structured predictor. In contrast, our work also learns the variance of the predictor, and our experimental validation shows it contributes to the performance of the predictor. Also, our theoretical contribution sets the framework to handle any structured predictor. 
Closely related is a method of differentiating through marginal inference \citep{NIPS2010_4107}, which shows that the gradient of the loss with respect to the parameters can be computed based on inference over the original parameters , and one over the parameters pertubed in the direction of the loss derivative w.r.t. to the marginals.

Another line of work considers continuous relaxations of the discrete structures.
\citet{Paulus2020GradientEW} have suggested a unified framework for constructing structured relaxations of combinatorial distributions, and have demonstrated it as a generalization of the Gumbel-Softmax trick. Their method builds upon differentiating through a convex program and induces solutions found in the interior of the polytope rather than on its faces, as a function of temperature-controlled approximation. An efficient extension for sorting and ranking differential operators has been suggested lately \citep{Blondel2020FastDS}.
SparseMAP \citep{Niculae2018SparseMAPDS} is a sparse structured inference framework which offers a continuous relaxation. It finds sparse MAP solutions on the faces of the marginal polytope.
Recently, \citet{berthet2020learning} suggested stochastic smoothing to allow differentiation
through perturbed maximizers.
In contrast, we do not use convex smoothing techniques of the structured label for differentiation.

Blackbox optimization \citep{DBLP:conf/iclr/PogancicPMMR20} is a new scheme to differentiate through argmax, which allows backward pass through blackbox implementations of combinatorial solvers with linear
objective functions. 

Our work considers two popular structured prediction problems: bipartite matching and $k$-nearest neighbors. Learning matchings in bipartite graphs has been extensively researched. When the bipartite graph is balanced, a matching can be represented by a permutation, which is an extreme point of the Birkhoff polytope, i.e., the set of all doubly stochastic matrices. Many works have built upon Sinkhorn normalization, an algorithm that maps a square matrix to a doubly-stochastic matrix. The Sinkhorn normalization has been incorporated in end-to-end learning algorithms in order to obtain relaxed gradients for learning to rank \citep{Adams2011RankingVS}, bipartite matching \citep{mena2018learning}, visual permutation learning \citep{8481554}, and latent permutation inference \citep{pmlr-v84-linderman18a}. This continuous relaxation is inspired by the Gumbel-Softmax trick \citep{Jang2016CategoricalRW, DBLP:conf/iclr/MaddisonMT17}. Andriyash et al. \citet{Andriyash2018ImprovedGO} have later showed that the Gumbel-Softmax estimator is biased and proposed a method to reduce its bias. We also consider the problem of stochastic
maximization over the set of possible latent permutations. However, we do not relax the use of bipartite matchings. Instead, we directly optimize the bipartite matching predictor and propagate gradients using the direct optimization approach.

Our work also considers learning $k$-nearest neighbors, i.e., learning an embedding of points that encourages the $k$ closest points to the test point to have the correct label. The body of work on sorting and specifically top-$k$ operators in an end-to-end learning framework is extensive.  \citet{grover2018stochastic} have suggested a continuous relaxation of the output of the sorting operator from permutation matrices to the set of unimodal row-stochastic matrices, where every row sums to one and has a distinct maximal argument. \citet{NIPS2018_7386} developed a continuous deterministic relaxation that maintains differentiability with respect to pairwise distances, but retains the original $k$-nearest neighbors as the limit of a temperature parameter approaching zero.
Other approaches are based on top-$k$ subset sampling \citep{DBLP:conf/ijcai/XieE19, conf/icml/KoolHW19}. \citet{berrada2018smooth} have introduce a family of smoothed, temperature controlled loss functions that are suited to top-k optimization. In contrast, our work does not relax the objective but rather directly optimize the top-$k$ neighbors. \citet{Xie2020DifferentiableTO} have proposed a smoothed approximation to the top-$k$ operator as the solution of an Entropic Optimal Transport problem.


\section{Background}
\label{sec:bg}

Learning to predict structured labels $y \in Y$ of data instances $x \in X$ covers a wide range of problems. The structure is incorporated into the label $y = (y_1,...,y_n)$ which may refer to matchings, permutations, sequences, or other high-dimensional objects. For any data instance $x$, its different structures are scored by a parametrized function $\mu_w(x,y)$. Discriminative learning aims to find a mapping from training data $S = \{(x_1,y_1),...,(x_m,y_m)\}$ to parameters $w$ for which $\mu_w(x,y)$ assign high scores to structures $y$ that describe well the data instance $x$. 
The parameters $w$ are fitted to minimize the loss $\ell(\cdot , \cdot)$ of the instance-label pairs $(x,y) \in S$ between the label $y$ and the highest scoring structure of $\mu_w(x,y)$. While gradient methods are the most popular methods to learn the parameters $w$, they are notoriously inefficient for learning discrete predictions. When considering discrete labels, the maximal argument of $\mu_w(x,y)$ is a piecewise constant function of $w$, and its gradient with respect to $w$ is zero for almost any $w$. Consequently, various smoothing techniques were proposed to propagate gradients while learning to fit discrete structures.

Direct loss minimization approach aims at minimizing the expected loss $\min_w \E_{(x,y) \sim {\cal D}} \ell(y^*_w, y)$ that incurs when the training label $y$ is different than the predicted label \citep{NIPS2010_4069,Keshet2011DirectER,Song2016TrainingDN}
\[
y^*_w  \triangleq \arg \max_{\hat y} \mu_w(x,\hat y) \label{eq:yopt} 
\]
Direct loss minimization relies on a loss-perturbed prediction 
\[
y^*_w(\epsilon) \triangleq \arg \max_{\hat y} \{ \mu_w(x,\hat y) + \epsilon \ell(y, \hat y)\}. \label{eq:ydirect} 
\]
It introduces an optimization-based gradient step for the expected loss, namely $\nabla \E_{(x,y) \sim {\cal D}} \ell(y^*_w, y) =$
\[
\lim_{\epsilon \rightarrow 0} \frac{1}{\epsilon} \Big( \E_{(x,y) \sim {\cal D}} [\nabla_w \mu_w(x, y^*_w(\epsilon)) - \nabla_w \mu_w(x, y^*_w)]\Big). \label{eq:direct} 
\]
Unfortunately, the above gradient step does not hold for any $w$, cf. \cite{NIPS2010_4069} Section 3.1. For example, when $w = 0$ the gradient estimator in Equation (\ref{eq:direct}) may be zero for any $(x,y) \sim {\cal D}$ regardless of the value of $\nabla \E_{(x,y) \sim {\cal D}} [\ell(y^*_w, y)]$. In Section \ref{sec:direct} we define the mathematical condition for which Equation (\ref{eq:direct}) represents the gradient. 

Recently, the direct loss minimization technique was applied to generative learning. In this setting, a random perturbation $\gamma(y)$ is added to each configuration, \citep{Lorberbom2018DirectOT}. The technique allows to randomly generate structures $y$ for any given $x$ from a generative distribution $q(y|x) \propto e^{\mu_w (x,y)}$. The generative learning approach relies on the connection between $q(y | x)$ and the Gumbel-max trick, namely $\P_{\gamma \sim g}[y^*_{w,\gamma} = y] \propto e^{\mu_w(x,y)}$, when $y^*_{w, \gamma} = \arg \max_{\hat y}  \{ \mu_w (x,\hat y) + \gamma(\hat y) \}$ and $\gamma(y)$ are i.i.d. random variables that follow the zero mean Gumbel distribution law, which we denote by ${\cal G}$. The corresponding gradient step, in discriminative learning setting, takes the form: $\nabla \E_{\gamma \sim {\cal G}} [\ell(y^*_{w,\gamma}, y)] =$
\[
\lim_{\epsilon \rightarrow 0} \frac{1}{\epsilon} \Big( \E_{\gamma \sim {\cal G}} [\nabla\mu_w(x, y^*_{w,\gamma}(\epsilon)) - \nabla \mu_w(x, y^*_{w,\gamma})]\Big). \label{eq:gdirect} 
\]
Here, $y^*_{w, \gamma}(\epsilon) = \arg \max_{\hat y} \{ \mu_w(x,\hat y) + \gamma(\hat y) + \epsilon \ell(y, \hat y)\}$. The advantage of using this framework in this setting is that it effortlessly elevates the mathematical difficulties in defining the gradient of the expected loss that exists in the direct loss minimization framework. Unfortunately, the random noise $\gamma(y)$ that is injected to the optimization may mask the signal $\mu_w(x,y)$ and thus get sub-optimal results in discriminative learning. To enjoy the best of both worlds, we propose to learn the proper amount of randomness to add to the discriminative learner. 

In this work we focus on learning discrete structured labels $y = (y_1,...,y_n)$. A general score function $\mu_w(x,y)$ cannot be computed efficiently for discrete structured labels $y = (y_1,...,y_n)$ since the number of possible labels is exponential in $n$ and a general score function $\mu_w(x,y)$ may assign a different value for each structure. Typically, such score functions are decomposed to localized score functions over small subsets $\alpha \subset \{1,...,n\}$ of variables where $y_\alpha = (y_i)_{i \in \alpha}$. The score function takes the form: $\mu_w(x,y) = \sum_{\alpha \in A} \mu_{w,\alpha}(x,y_\alpha)$. The correspondence between the exponential family of distributions $e^{\mu_w(x,y)}$ and the Gumbel-max trick requires an independent random variable $\gamma(y)$ for each of the exponentially many structures $y = (y_1,.,,,.y_n)$. However, since we are focusing on discriminative learning we are not limited by the Gumbel-max trick. Instead, we can use fewer random variables in order to learn the minimal amount of randomness to add. We limit our predictors to low-dimensional independent random variables $\gamma(y) = \sum_{i=1}^n \gamma_i(y_i)$, where $\gamma_i(y_i)$ are independent random variables for each index $i = 1,...,n$ and each $y_i$. In this setting, the number of random variables we are using is linear in $n$, compared to exponential many random variables in the Gubeml-max setting.


\section{Learning Structured Predictors}

In the following we present our main technical concept that derives the gradient of an expected loss using two structured predictions. In Section \ref{sec:direct} we prove the gradient step of an expected loss in the direct loss minimization framework. We also deduce that it holds whenever $y_w^*(\epsilon)$ is unique. Subsequently, in Section \ref{sec:direct_perturbed}, we show that low dimensional random perturbations $\gamma_i(y_i)$ are able to implicitly enforce uniqueness of the maximizing structure with probability one. In Section \ref{sec:direct_perturbed_snr}, we present our approach that learns the variance of the random perturbation, to ensure that the random noise $\gamma_i(y_i)$ does not mask the signal $\mu_w(x,y)$.

\subsection{Direct Loss Minimization}  
\label{sec:direct} 

We rely on the expected max-value that is perturbed by the loss function. This is the ``prediction generating function" in \citet{Lorberbom2018DirectOT}. In the direct loss minimization setting, as defined in Equation (\ref{eq:direct}), this function takes the form: 
\[
G(w,\epsilon) = \mathbb{E}_{{(x,y) \sim {\cal D}}} \Big[ \max_{\hat y \in Y} \big\{\mu_{w}(x, \hat y)  + \epsilon \ell(y, \hat y) \big\} \Big]  \label{eq:G_pred_direct} 
\]
The proof technique relies on the existence of the Hessian of $G(w,\epsilon)$ and the main challenge is to show that $G(w,\epsilon)$ is differentiable, i.e., there exists a vector $\nabla \mu_{w}(x, y^*(\epsilon))$ such that for any direction $u$, its corresponding directional derivative $\lim_{h \rightarrow 0} \frac{G(w + hu, \epsilon) - G(w,\epsilon)}{h}$ equals $\E_{(x,y) \sim {\cal D}}[\nabla_w \mu_{w}(x, y^*(\epsilon))^\top u]$. The proof builds a sequence of functions $\{g_n(u)\}_{n=1}^\infty$ that satisfies   
\begin{equation}
\lim_{h \rightarrow 0} \frac{G(w + hu, \epsilon) - G(w,\epsilon)}{h} = \lim_{n \rightarrow \infty} \E_{(x,y) \sim {\cal D}}[g_n(u)] \label{eq:grad1} 
\end{equation}
\begin{equation}
\E_{(x,y) \sim {\cal D}}[\lim_{n \rightarrow \infty}  g_n(u)] =\E_{(x,y) \sim {\cal D}}[\nabla_w \mu_{w}(x, y^*(\epsilon))^\top u]. \label{eq:grad2}
\end{equation}
The functions $g_n(u)$ correspond to the loss perturbed prediction $y_w^*(\epsilon)$ through the quantity $\mu_{w + \frac{1}{n} u}(x, \hat y)  + \epsilon \ell(y, \hat y)$. The key idea we are exploiting is that there exists $n_0$ such that for any $n \ge n_0$ the maximal argument $y_{w + \frac{1}{n} u}^*(\epsilon)$  does not change. 
\begin{lemma}
\label{lemma:maxfn}
Assume $\mu_w(x,y)$ are continuous functions of $w$ and that their loss-perturbed maximal argument $y^*_{w + \frac{1}{n} u}(\epsilon)$, which is defined in Equation (\ref{eq:ydirect}), is unique for any $u$ and $n$. Then there exists $n_0$ such that for $n \ge n_0$ there holds $y^*_{w + \frac{1}{n} u}(\epsilon) = y_w^*(\epsilon)$. 
\end{lemma}
\begin{proof}
Let $f_n(y) = \mu_{w + \frac{1}{n} u}(x, y)  + \epsilon \ell(y, \hat y)$ so that $y^*_{w + \frac{1}{n} u}(\epsilon) = \arg \max_y f_n(y)$. Also, let $f_\infty(y) = \mu_w(x, y)  + \epsilon \ell(y, \hat y)$ so that $y^*_{w}(\epsilon) = \arg \max_y f_\infty(y)$. Since $f_n$ is a continuous function of then $\max_y  f_n(y)$ is also a continuous function and $\lim_{n \rightarrow \infty} \max_y  f_n(y) = \max_y  f_\infty(y)$. Since $\max_y  f_n(y) = f_n(y^*_{w + \frac{1}{n} u}(\epsilon))$ is arbitrarily close to $\max_y f_\infty(y) = f_\infty(y^*_{w}(\epsilon))$, and $y^*_w(\epsilon), y^*_{w + \frac{1}{n} u}(\epsilon)$ are unique then for any $n \ge n_0$ these two arguments must be the same, otherwise there is a $\delta > 0$ for which $| f_\infty(y^*_{w}(\epsilon)) - f_n(y^*_{w + \frac{1}{n} u}(\epsilon)) | \ge \delta$. 
\end{proof}
This lemma relies on the discrete nature of the label space, ensuring that the optimal label does not change in the vicinity of $y^*_w(\epsilon)$. This phenomena distinguishes the discrete label setting from the continuous relaxations of the label space \citep{NIPS2010_4107,berthet2020learning, Paulus2020GradientEW}. These relaxations of the label space utilize their continuities to differentiate through the label. In direct loss minimization, one works directly with the discrete label space which allows to control the maximal argument in infinitesimal interval. 

\begin{theorem}\label{main_theorem_direct}
Assume $\mu_{w}(x,y)$ is a smooth function of $w$ and that $E_{(x,y) \sim {\cal D}}\|\nabla_w \mu_{w}(x,y)\| \le \infty$. 
If the conditions of Lemma \ref{lemma:maxfn} hold then the prediction generating function $G(w,\epsilon)$, as defined in Equation (\ref{eq:G_pred_direct}), is differentiable and 
\[
    \frac{\partial G(w,\epsilon)}{\partial \epsilon} &=& \mathbb{E}_{(x,y) \sim {\cal D}}[\ell(y, y^*_{w})]. \\
    \frac{\partial G(w,\epsilon)}{\partial w} &=& \E_{(x,y) \sim {\cal D}} \Big[\nabla \mu_{w}(x, y^*(\epsilon)) \Big].
\]
\end{theorem}

\begin{proof}
Let $f_n(y) = \mu_{w + \frac{1}{n} u}(x, y)  + \epsilon \ell(y, \hat y)$ as in Lemma \ref{lemma:maxfn} and let 
\[
    g_n(u) \triangleq \frac{\max_{\hat y \in Y} f_n( \hat y) - \max_{\hat y \in Y} f_\infty(\hat y)}{1/n} \label{eq:gn_direct}
\]
We apply the dominated convergence theorem on $g_n(u)$, so that $ \lim_{n \rightarrow \infty} \E_{(x,y) \sim {\cal D}} [g_n(u)] = \E_{(x,y) \sim {\cal D}} [\lim_{n \rightarrow \infty}  g_n(u)]$ in order to prove Equations (\ref{eq:grad1},\ref{eq:grad2}). 
We note that we may apply the dominated convergence theorem, since the conditions $\E_{(x,y) \sim D} \|\nabla_w \mu_{w}(x,y)\| \le \infty$ imply that the expected value of $g_n$ is finite (We recall that $f_n$ is a measurable function, and note that since $\hat y \in Y$ is an element from a discrete set $Y$, then $g_n$ is also a measurable function.).

From Lemma \ref{lemma:maxfn}, the terms $ \ell(y, y^*(\epsilon))$ are identical in both $\max_{\hat y \in Y} f_n(\hat y) $ and $\max_{\hat y \in Y} f_\infty(\hat y)$. Therefore, they cancel out when computing the difference $\max_{\hat y \in Y} f_n(\hat y)  - \max_{\hat y \in Y} f_\infty(\hat y) $. Then, for $n \ge n_0$:
\begin{equation*}
\max_{\hat y \in Y} f_n(\hat y) - \max_{\hat y \in Y} f_\infty(\hat y) = \mu_{w + \frac{1}{n} u}(x, y^*(\epsilon))
-\mu_{w}(x, y^*(\epsilon))
\end{equation*}
and Equation (\ref{eq:gn_direct}) becomes:
\begin{equation}
g_n(u) = \frac{ \mu_{w + \frac{1}{n} u}(x, y^*(\epsilon)) - \mu_{w}(x, y^*(\epsilon))}{1/n} \label{gn_max_f_direct}.
\end{equation}
Since $\mu_{w}(x, y^*(\epsilon))$ is smooth, then $\lim_{n \rightarrow \infty} g_n(u) $ is composed of the derivatives of $\mu_{w}(x, y^*(\epsilon))$ in direction $u$, namely, $\lim_{n \rightarrow \infty} g_n(u)= \nabla_w \mu_{w}(x, y^*(\epsilon))^\top u$. 
\end{proof}
In the above theorem we assume that $\mu_w(x,y)$ is smooth, namely it is infinitely differentiable. It suffices to assume that $\mu_w(x,y)$ is twice differentiable, to ensure that $G(w,\epsilon)$ is twice differentiable and hence its Hessian exists.

\begin{corollary}\label{directloss_gradient_cor}
Under the conditions of Theorem \ref{main_theorem_direct},  $\mathbb{E}_{(x,y) \sim {\cal D}}[\ell(y, y^*_{w})]$ is differentiable and its derivative is defined in Equation (\ref{eq:direct}).
\end{corollary}

\begin{proof}
Since Theorem \ref{main_theorem_direct} holds for every direction $u$:
\begin{equation*}
\frac{\partial G(w,\epsilon)}{\partial w} = \E_{(x,y) \sim \cal D} \Big[\nabla_w \mu_{w}(x, y^*(\epsilon)) \Big].
\end{equation*}
Adding a derivative with respect to $\epsilon$ we get: 
\begin{flalign*}
\frac{\partial}{\partial \epsilon} \frac{\partial G(w,0)}{\partial w} &=\\
\lim_{\epsilon \rightarrow 0} \frac{1}{\epsilon} \E_{(x,y) \sim \cal D} \Big[& \nabla_w \mu_{w}(x, y^*(\epsilon)) - \nabla_w \mu_{w}(x, y^*)  \Big]
\end{flalign*}
The proof follows by showing that the gradient computation is apparent in the Hessian, namely Equation (\ref{eq:direct}) is attained by the identity $\partial_w \partial_\epsilon G(w,0) = \partial_\epsilon \partial_w G(w,0)$. 
Now we turn to show that $\partial_w \partial_\epsilon G(w,0) = \nabla_{w} \mathbb{E}_{(x,y) \sim \cal D}[\ell(y, y^*_{w})]$. Since $\epsilon$ is a real valued number rather than a vector, we do not need to consider the directional derivative, which greatly simplifies the mathematical derivations. We define $f_n(\hat y) \triangleq   \mu_{w}(x, \hat y) + \frac{1}{n} \ell(y, \hat y)$ and follow the same derivation as above to show that $\partial_\epsilon G(w,0) = \mathbb{E}_{(x,y) \sim \cal D}[\ell(y, y^*_{w})]$. Therefore $\partial_w \partial_\epsilon G(w,0) = \nabla_{w} \mathbb{E}_{(x,y) \sim \cal D}[\ell(y, y^*_{w})]$.
\end{proof}

We note the strong conditions that require the theorem to hold: the loss-perturbed maximal argument $y^*_{w + \frac{1}{n} u}(\epsilon)$, which is defined in Equation (\ref{eq:ydirect}), is unique for any $u$ and $n$. Unfortunately, this condition does not hold in some cases, e.g., when $w = 0$. Next we show that with added random perturbation we can ensure this holds with probability one. 

\subsection{Randomly Perturbing Structured Predictors} 
\label{sec:direct_perturbed}

We turn to show that randomly perturbing the structured signal $\mu_w(x,y) = \sum_{\alpha \in {\cal A}} \mu_\alpha(x,y_\alpha)$ with smooth random noise $\gamma_i(y_i)$ allows us to implicitly enforce the uniqueness condition. To account for the structured signal and the low-dimensional random perturbation we define the set $y^*_{w, \gamma}(\epsilon) =$
\[
 \arg \max_{\hat y} \Big\{ \sum_{\alpha \in {\cal A}} \mu_{w,\alpha}(x,\hat y_\alpha) + \sum_{i=1}^n \gamma_i(\hat y_i) + \epsilon \ell(y, \hat y) \Big\}. \label{eq:argmax}
\]
To reason about the set of maximal structures of $y^*_{w, \gamma}(\epsilon)$, we introduce the set of random perturbation $\Gamma_\epsilon(y')$ which consists of all random values $\gamma_i(y_i)$ for which $y'$ is their maximal structure:  $\Gamma_\epsilon(y') = $
\begin{equation}
\left\{ 
{\everymath={\displaystyle}
\begin{array}{lll} 
\gamma &:&  \sum_{\alpha \in {\cal A}} \mu_{w,\alpha}(x,\hat y'_\alpha) + \sum_{i=1}^n \gamma_i(\hat y'_i) +\epsilon \ell(y, \hat y') \\
&&\ge  \\
&& \forall \hat y  \sum_{\alpha \in {\cal A}} \mu_{w,\alpha}(x,\hat y_\alpha) + \sum_{i=1}^n \gamma_i(\hat y_i) +\epsilon \ell(y, \hat y) \\
\end{array}}
\right\}
\end{equation}

Whenever the set in Equation (\ref{eq:argmax}) consists of more than a single structure, say $y'$ and $y''$, their corresponding sets $\Gamma_\epsilon(y')$ and $\Gamma_\epsilon(y'')$ intersect. We now prove that this happens with zero probability whenever $\gamma_i(y_i)$ are i.i.d. and with a smooth probability density function.

\begin{theorem}\label{uniqueness}
Let $\gamma_i(y_i)$ be i.i.d. random variables with a smooth probability density function. Then the set of maximal arguments in Equation (\ref{eq:argmax}) consists of a single structure with probability one. 
\end{theorem}

\begin{proof}
Consider there is an event (a set) of $\gamma$ for which the set of maximal arguments consists of more than one structure, e.g., $y'$ and $y''$ and denote it by $\Omega$. Clearly, $\Omega \subset \Gamma_\epsilon(y') \cap \Gamma_\epsilon(y'')$. Let $\beta$ be the set of indexes for which $y'_i \ne y''_i$. Since for any $\gamma \in \Omega$ it holds that $\sum_{\alpha \in {\cal A}} \mu_{w,\alpha}(x,\hat y'_\alpha) + \sum_{i=1}^n \gamma_i(\hat y'_i) +\epsilon \ell(y, \hat y') = \sum_{\alpha \in {\cal A}} \mu_{w,\alpha}(x,\hat y''_\alpha) + \sum_{i=1}^n \gamma_i(\hat y''_i) +\epsilon \ell(y, \hat y'')$, then $\Omega \subset \{\gamma: \sum_{i \in \beta} \gamma_i(\hat y'_i) - \gamma_i(\hat y''_i) = c\}$ for the constant $c = \mu_{w,\alpha}(x,\hat y''_\alpha) - \mu_{w,\alpha}(x,\hat y'_\alpha) + \epsilon \ell(y, \hat y'')- \epsilon \ell(y, \hat y')$. Since $\gamma_i(\hat y'_i) - \gamma_i(\hat y''_i)$ are independent random variables with smooth probability density function, then their sum also has a smooth probability density function. Consequently the probability that $\sum_{i \in \beta} \gamma_i(\hat y'_i) - \gamma_i(\hat y''_i) = c$ is zero, and thus $\P_\gamma[\Omega] = 0$. 
\end{proof}

We note that the uniqueness of the maximal structure of $y^*_{w,\gamma}$ can be proved by Theorem \ref{uniqueness} as well, in which case the constant $c = \mu_{w,\alpha}(x,\hat y''_\alpha) - \mu_{w,\alpha}(x,\hat y'_\alpha)$.
It follows from the above theorem that adding random perturbations solves the uniqueness problem in direct loss gradient rule. Unfortunately, as we show in our experimental evaluation, the random perturbation that smooths the objective can also serve as noise that masks the signal $\mu_w(x,y)$. To address this caveat, we propose to learn the magnitude, i.e., the variance, of this noise explicitly. 

\subsection{Learning The Variance Of Randomly Perturbed Structured Predictors}   
\label{sec:direct_perturbed_snr}

\begin{figure*}[t]
\begin{subfigure}[t]{\columnwidth}
    \includegraphics[width=1.0\linewidth]{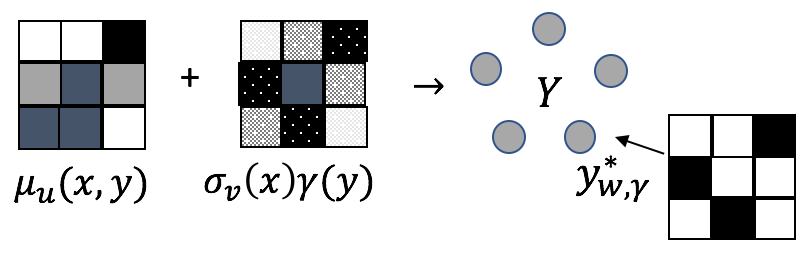}
    \caption{A randomly perturbed structured prediction illustration}
  \label{structured_perturb_illustration}
  \end{subfigure}
  \begin{subfigure}[t]{.51\columnwidth}
    \includegraphics[width=\linewidth]{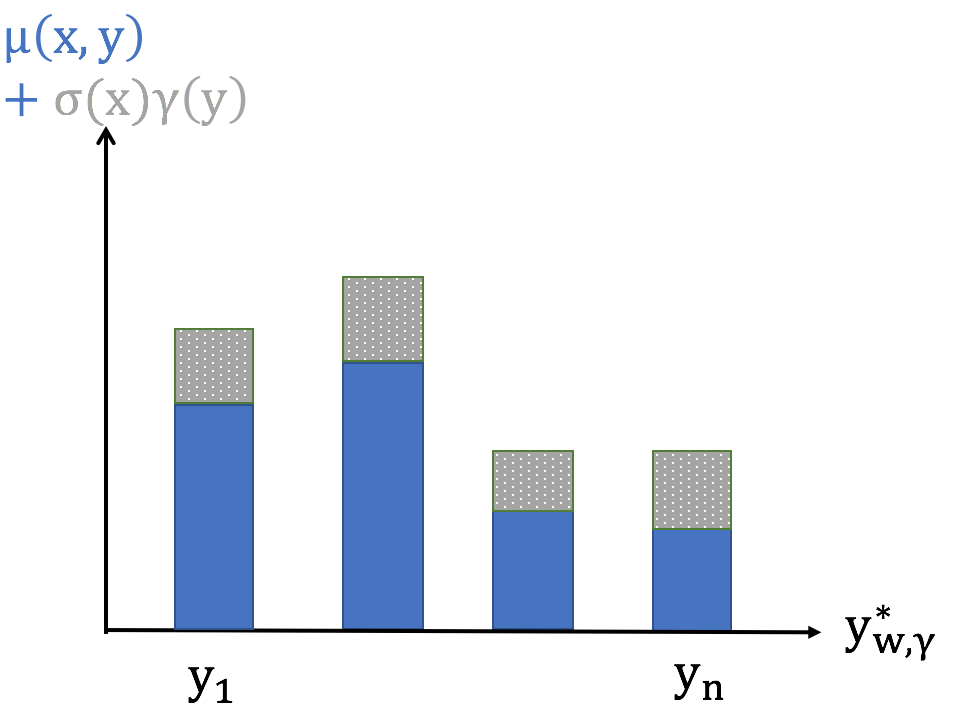}
    \caption{The max predictor probability distribution for a low $\sigma(x)$.}
  \label{low_sigma}
  \end{subfigure}%
  \hfill
  \begin{subfigure}[t]{.51\columnwidth}
    \includegraphics[width=\linewidth]{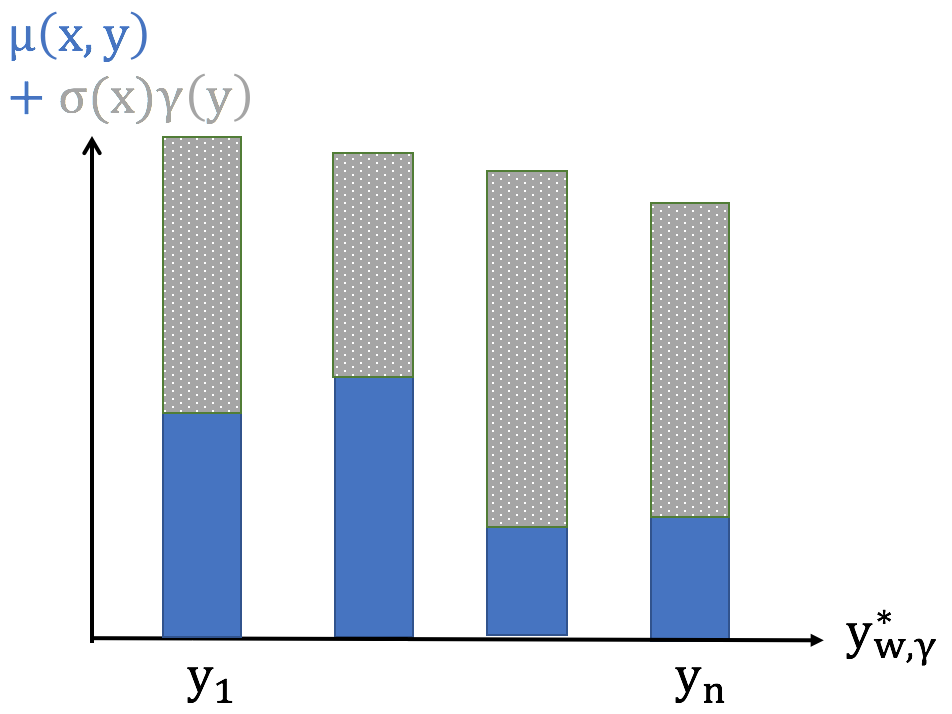}
    \caption{The max predictor probability distribution for a high $\sigma(x)$.}
      \label{high_sigma}
  \end{subfigure}
\caption{The randomized predictor $y^*_{w,\gamma}$ is the structure that maximizes the randomly perturbed scoring function among all possible structures in $Y$ (Figure \ref{structured_perturb_illustration}). As $\sigma(x)$ decreases, the expected max predictor approaches the expected value of a categorical random variable (Figure \ref{low_sigma}). And vice versa, as $\sigma(x)$ increases, the expected max predictor converges to a uniform distribution over the discrete structures (Figure \ref{high_sigma}).} 
\label{sigma_smoothness}
\vspace{-\baselineskip}
\end{figure*}

We propose to learn the magnitude of the random perturbation $\gamma_i(y_i)$. In our setting it translates to the prediction $y^*_{w,\gamma}=$
\begin{equation}
 \arg \max_{\hat y} \Big\{ \sum_{\alpha \in {\cal A}} \mu_{u,\alpha}(x,\hat y_\alpha) + \sum_{i=1}^n \sigma_v(x) \gamma_i(\hat y_i) \Big\} 
\end{equation}
$w = (u,v)$ are the learned parameters. In this case we treat $\sum_{\alpha \in {\cal A}} \mu_{u,\alpha}(x,\hat y_\alpha) + \sum_{i=1}^n \sigma_v(x) \gamma_i(\hat y_i)$ as a random variable whose mean is learned using $\mu_{u,\alpha}(x,\hat y_\alpha) $ and its variance is learned using $\sigma_{v}(x)$. As such, we consider a strictly positive $\sigma_{v}(x)$ both theoretically and practically.
We are learning the same variance $\sigma_{v}(x)$ for all random assignments $\gamma_i(y_i)$. We do so to learn to balance the overall noise $\sum_{i=1}^n \gamma_i(y_i)$ with the signal $\sum_{\alpha} \mu_{u,\alpha}(x,\hat y_\alpha)$. The learned variance $\sigma_v(x)$ allows us to interpolate between the original direct loss setting, where $\sigma_v(x)=0$, to the generative learning setting, where $\sigma_v(x)=1$.      

\begin{corollary}\label{corollary_DSL}
Assume $\mu_{u}(x,y), \sigma_v(x)$ are smooth functions of $w = (u,v)$. Let $\gamma_i(y_i)$ be i.i.d. random variables with a smooth probability density function. Let $G(w,\epsilon) =$
\begin{equation}
\mathbb{E}_{\gamma \sim {\cal G}} \Big[   \max_{\hat y \in Y} \Big\{ \sum_{\alpha \in {\cal A}} \mu_{u,\alpha}(x,\hat y_\alpha) + \sum_{i=1}^n \sigma_v(x) \gamma_i(\hat y_i) + \epsilon \ell(y, \hat y)  \Big\}  \Big].  
\end{equation}
Then $G(w,\epsilon)$ is a smooth function and $\frac{\partial}{\partial u} \mathbb{E}_{\gamma}[\ell(y, y^*_{w, \gamma})] =$ 
\[
\lim_{\epsilon\to 0} \frac{1}{\epsilon} \mathbb{E}_{\gamma}\Big[ \sum_{\alpha \in {\cal A}} \left( \nabla \mu_{u,\alpha}(x,y^*_\alpha(\epsilon)) - \nabla \mu_{u,\alpha}(x,y^*_\alpha) \right)  \Big] \label{eq:direct_mu} \]
and  $\frac{\partial}{\partial v} \mathbb{E}_{\gamma}[\ell(y, y^*_{w, \gamma})] =$ 
\[ 
\lim_{\epsilon\to 0} \frac{1}{\epsilon} \mathbb{E}_{\gamma} \Big[\sum_{i=1}^n  \nabla \sigma_{v}(x) \Big(\gamma_i (y^*_i(\epsilon)) - \gamma_i(y^*_i)\Big)\Big] .  \label{eq:direct_sigma}
\]
\end{corollary}

We prove Corollary \ref{corollary_DSL} in the supplementary material.

The random perturbation induces a probability distribution over structures $y$. As $\sigma$ increases, the expected max predictor tends to a uniform distribution over the discrete structures. Similarly, as $\sigma$ decreases, the expected max predictor approaches a deterministic decision over the discrete structures. This idea is illustrated in Figure \ref{sigma_smoothness}.

Interestingly, whenever the random variables $\gamma(y)$ follow the zero mean Gumbel distribution law, the random variable $\mu_u (x,\hat y) + \sigma_v(x) \gamma(\hat y)$ follows the Gumbel distribution law with mean $\mu_u(x,y)$ and variance $\sigma^2 \pi^2 / 6$. In this case, the variance turns to be the temperature of the corresponding Gibbs distribution: $\P_{\gamma \sim {\cal G}}[\arg \max_{\hat y} \{ \mu_u (x,\hat y) + \sigma_v(x) \gamma(\hat y) \} = y] \propto e^{\mu_u(x,y)/\sigma_v(x)}$, see proof in the supplementary material. Our framework thus also allows to learn the temperature of the Gumbel-max trick instead of tuning it as a hyper-parameter.

\section{Experimental Validation}

In the following we validate the advantage of our approach (referred to as `Direct Stochastic Learning') in two popular structured prediction problems: bipartite matching and $k$-nearest neighbors. We compare to direct loss minimization \citep{NIPS2010_4069}, which can be interpreted as setting the noise variance to zero (referred to as Direct $\bar{\sigma} = 0$), as well as to \citet{Lorberbom2018DirectOT}, in which the noise variance is set to one (referred to as Direct $\bar{\sigma} = 1$). Additionally, we compare to   state-of-the-art in neural sorting \citep{grover2018stochastic,DBLP:conf/ijcai/XieE19} and bipartite matching \citep{mena2018learning}. Further architectural and training details are described in the supplementary material.

In all direct loss based experiments we set a negative $\epsilon$. When $\epsilon > 0$ the loss-pertubed label $y_w^*(\epsilon)$ chooses a configuration with a higher loss and performs a gradient descent step on $\nabla_w \mu_w(x, y^*_w(\epsilon))$, i.e., it moves the parameters $w$ to reduce the score function for the high-loss label $\mu_w(x, y^*_w(\epsilon))$. When $\epsilon < 0$  the loss-perturbed label $y_w^*(\epsilon)$ chooses a configuration with a lower loss and performs a gradient descent step on 
$-\nabla_w \mu_w(x, y^*_w(\epsilon))$, i.e., it increases the score function for the low-loss label $\mu_w(x, y^*_w(\epsilon))$. This choice is especially important in structured prediction, when there might be exponentially many structures with high loss and only few structures with low loss. This observation already appears in the original direct loss work \citep{NIPS2010_4069} (see last paragraph of Section 2).

\subsection{Bipartite Matchings}
We follow the problem setting, architecture $\mu_{u,\alpha}(x,y_\alpha)$ and loss function $\ell(y,y^*)$ of \citet{mena2018learning} for learning bipartite matching, and replace the Gumbel-Sinkhorn operation with our gradient step, see Figure \ref{PermutationLearningDiagram}.

\begin{figure*}[t]
  \centering
    \includegraphics[width=0.68\textwidth]{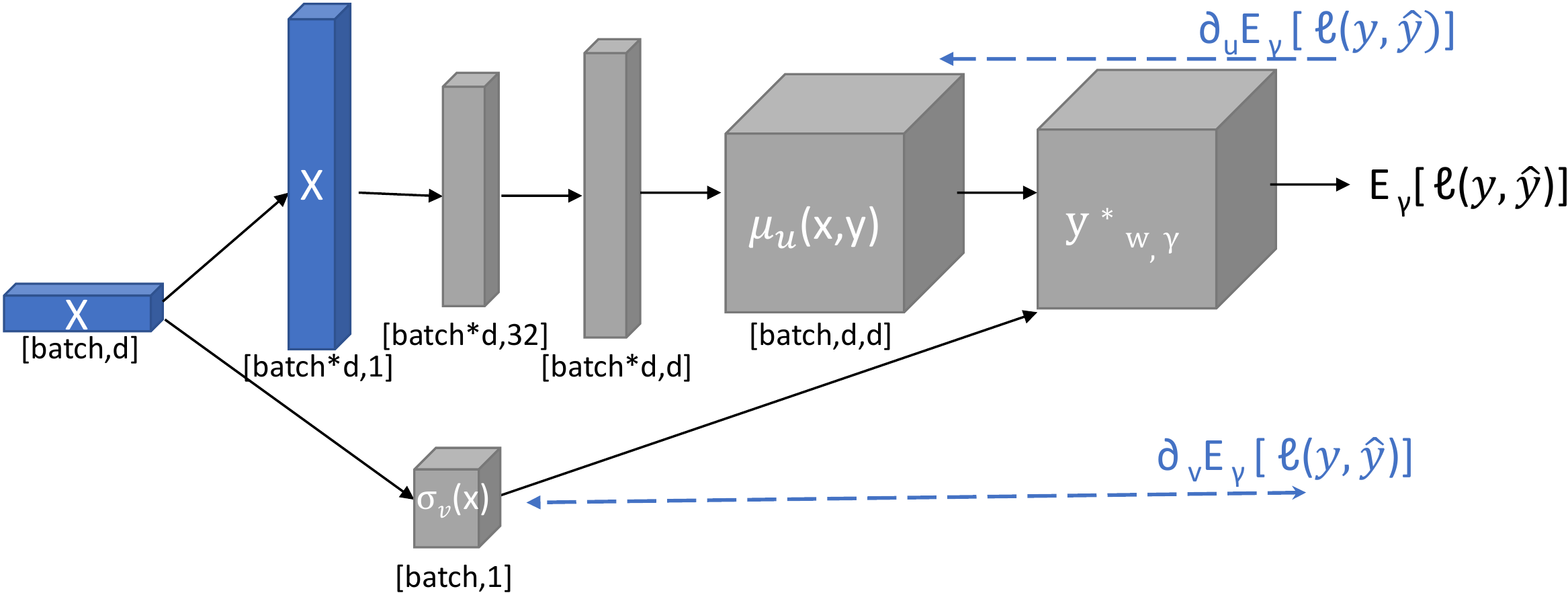}
    \caption{Architecture for learning bipartite matchings: The expectancy over Gumbel noise of the loss is derived w.r.t. the parameters $u$ of the signal and w.r.t. the parameters $v$ of the variance controller $\sigma$ directly (Equations \ref{eq:direct_mu},\ref{eq:direct_sigma} respectively).
    The network $\mu$ has a first fully connected layer that links the  sets of samples to an intermediate representation (with 32 neurons), and a second (fully connected) layer that turns those representations into batches of latent permutation matrices of dimension $d$ by $d$ each. It has the same architecture as the equivalent experiment by \citet{mena2018learning}. The network $\sigma$ has a single layer connecting input sample sequences to a single output which is then activated by a softplus activation. We chose such an activation to enforce a positive $\sigma$ value.}
\label{PermutationLearningDiagram}
\vspace{-\baselineskip}
\end{figure*}

In this experiment each training example $(x,y) \in S$ consists of an input vector $x \in R^d$ of $d$ numbers drawn independently from the uniform distribution over the $[0,1]$ interval. The structured label $y$, $y \in \{0,1\}^{d^2}$, is a bipartite matching between the elements of $x$ to the elements of the sorted vector of $x$. Formally, $y_{ij} = 1$ if $x_i = \mbox{sort}(x)_j$ and zero otherwise. Here we set $\alpha$ to be the pair of indexes $i,j = 1,\dots, d$ that corresponds to the desired bipartite matching. The network learns a real valued number for each $(i,j)$-th entry, namely, $\mu_{u,ij}(x,y_{ij})$ and our gradient update rule in Equation (\ref{eq:direct_mu}) replaces the Gumbel-Sinkhorn operator of \citet{mena2018learning}. We note that $y^*_{w,\gamma}$ can be computed efficiently  using any max-matching algorithm, which maximizes a linear function over the set of possible matching $Y$: $y^*_{w, \gamma} =$ 
\begin{equation}
\arg \max_{\hat y \in Y}  \sum_{ij=1}^d \mu_{u,ij} (x,\hat y_{ij})
+\sum_{ij=1}^d \sigma_v(x) \gamma_{ij} (\hat y_{ij})
\end{equation}
Our gradient computation also requires the loss perturbed predictor $y^*_{w, \gamma}(\epsilon)$, which takes into account the quadratic loss function:
\begin{equation*}
\ell(y,\hat y) = \sum_{i=1}^d \left(\sum_{j=1}^d x_j y_{ij} - \sum_{j=1}^d x_j \hat y_{ij} \right)^2
\end{equation*}
Note that this loss function is not smooth, as it is a function of binary elements, i.e.\ the squared differences between one hot vectors.
Seemingly, the quadratic loss function does not decompose along the score structure $\mu_{u,ij}(x,y_{ij})$, therefore it is challenging to recover the loss-perturbed prediction efficiently. Instead, we use the fact that $y_{ij}^2 = y_{ij}$ for $y_{ij} \in \{0,1\}$ and represent the loss as a linear function over the set of all matchings: $\ell(y,\hat y) = t + \sum_{j=1}^d t_{ij} \hat y_{ij}$, with $t = \sum_{ij=1}^d x_j^2 y_{ij}$ and $t_{ij} = \sum_{i=1}^d x_j^2 (1-2y_{ij})$. (Further details in the supplementary material). With this, we are able to recover the loss-perturbed predictor $y^*_{w,\gamma}(\epsilon)$ with the same computational complexity as $y^*_{w,\gamma}$, i.e., using linear solver over maximum matching:
\[
    y^*_{w,\gamma}(\epsilon) &=& \arg \max_{\hat y \in Y}  \sum_{ij=1}^d \mu_{u,ij} (x,\hat y_{ij}) \\
    &+& \sum_{ij=1}^d \sigma_v(x) \gamma_{ij} (\hat y_{ij})  + \epsilon \sum_{ij=1}^d t_{ij} \hat y_{ij} \nonumber.
\]
Note that we may omit $t$ from the optimization since it does not impact the maximal argument. 

In our experimental validation we found that negative $\epsilon$ works the best. In this case, when $y_{ij} = 0$ the corresponding embedding potential $\mu_{u,ij}(x,y_{ij})$ is perturbed by $-|\epsilon| x_j^2$, while when $y_{ij} = 1$ it is increased by $|\epsilon| x_j^2$. Doing so incrementally pushes $y^*_{w,\gamma}(\epsilon)$ towards predicting the ground truth permutation, which is aligned with our intuition of the towards-best direct loss minimization. The dynamics of our method as a function of matching dimension under a positive versus a negative $\epsilon$ is illustrated in Figure \ref{fig:dynamics}. In Figure \ref{fig:dynamics_loss} we plot the loss as a function of training epochs with varying size of matching dimension $d$. While the loss is similar for $\epsilon > 0$ and $\epsilon < 0$ when $d=10$, this changes for $d=100$. As such, the percentage of correctly sorted input entries of $y^*_w$ and $y^*(\epsilon)$ greatly differs when $d=100$ for different $\epsilon$. Importantly, when learning with $\epsilon > 0$ there are less than 40\% correct entries in $y^*_w$ (Figure \ref{fig:dynamics_e_positive}), while when learning with $\epsilon < 0$ there are at least $90\%$ correct entries in $y^*_w$ (Figure \ref{fig:dynamics_e_negative}).    

\begin{figure*}[t]
\vspace{-0.1cm}
    \centering
     \begin{subfigure}[b]{51mm}
         \centering \includegraphics[width=\textwidth]{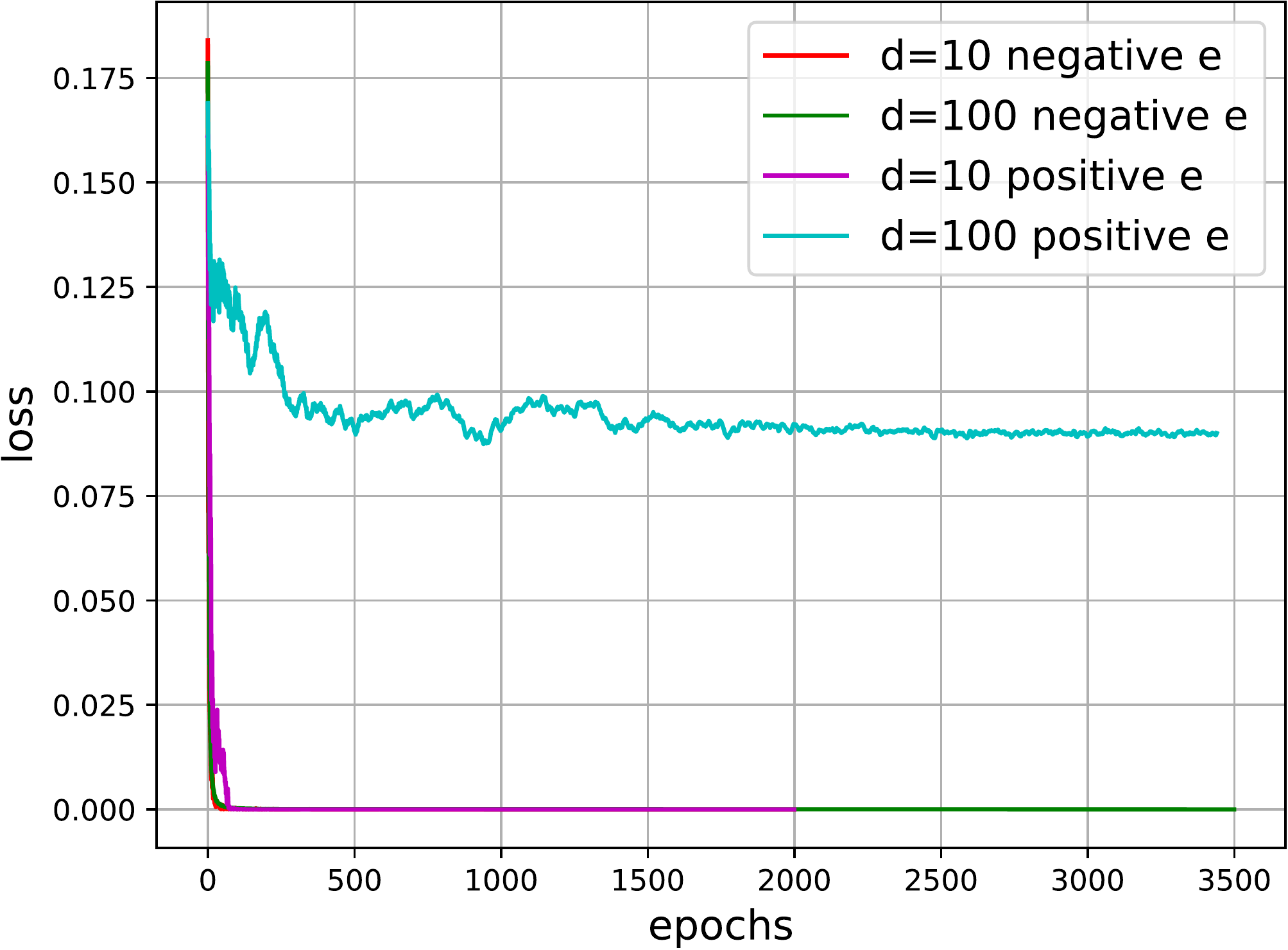}
         \caption{loss with negative and positive $\epsilon$}
         \label{fig:dynamics_loss}
     \end{subfigure}
     \hfill
     \begin{subfigure}[b]{51mm}
         \centering
         \includegraphics[width=\textwidth]{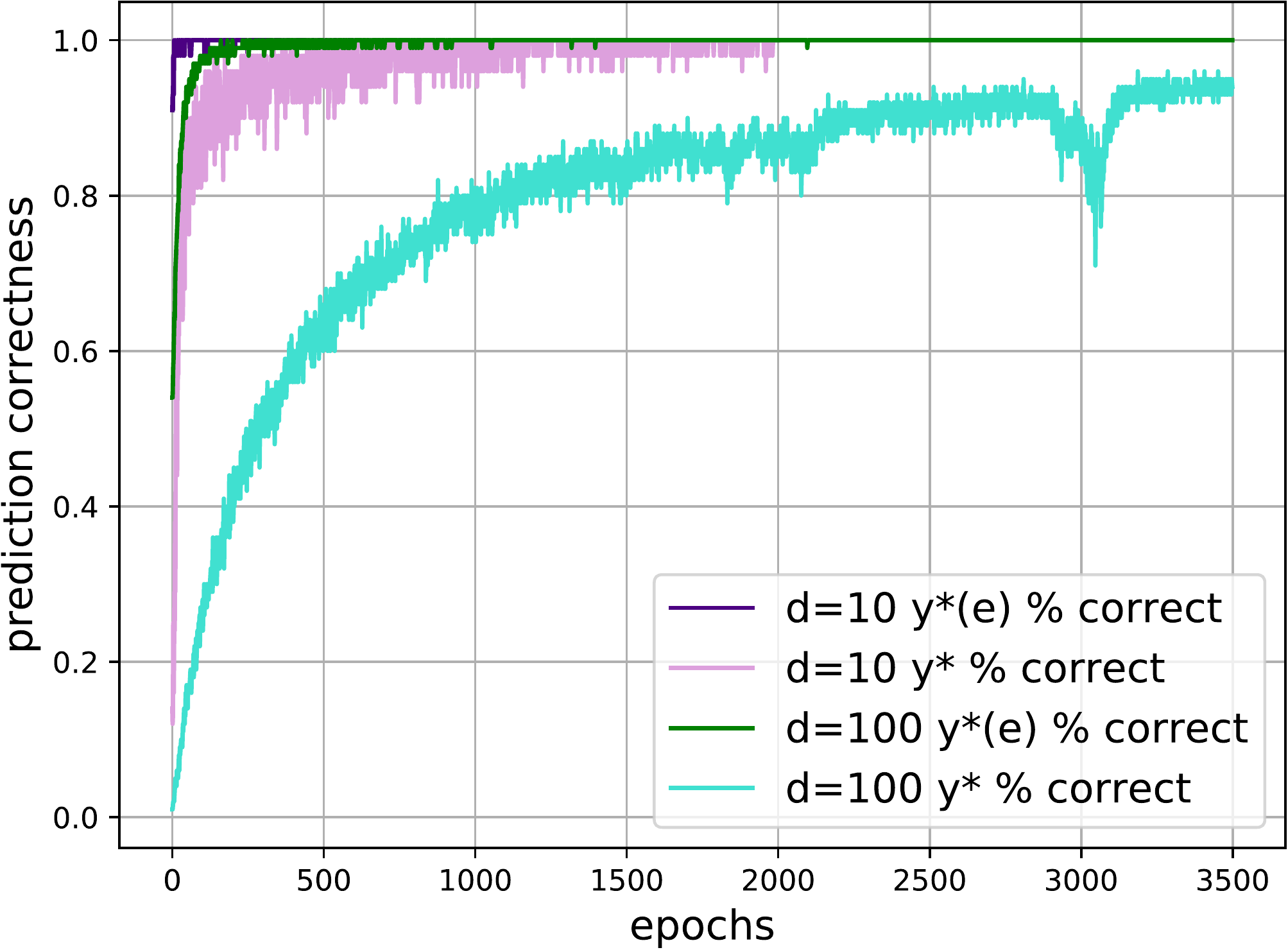}
         \caption{$y^*$ and $y^*(\epsilon)$ with negative $\epsilon$}
         \label{fig:dynamics_e_negative}
     \end{subfigure}%
    \hfill
    \begin{subfigure}[b]{51mm}
        \centering
        \includegraphics[width=\textwidth]{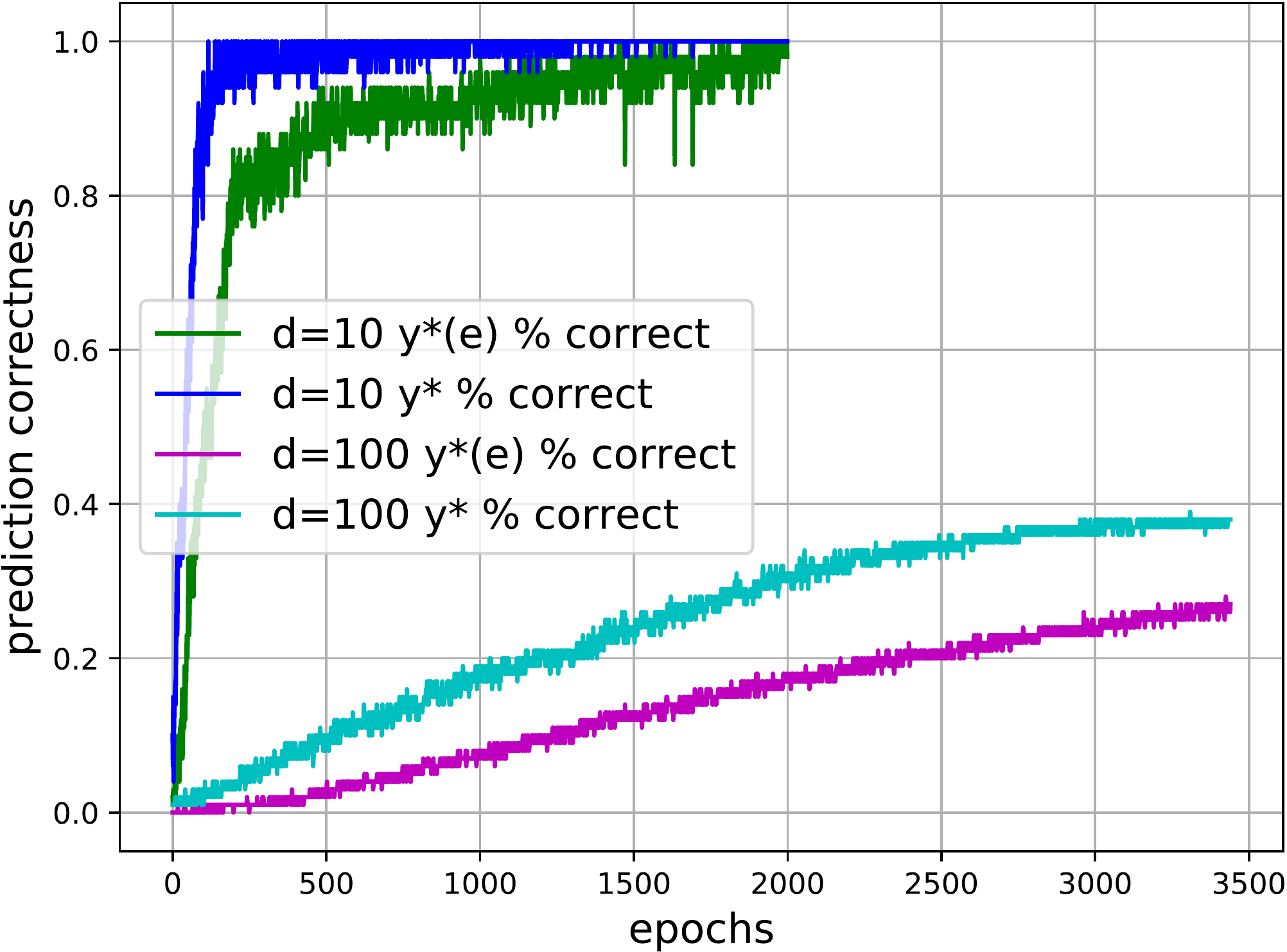}
        \caption{$y^*$ and $y^*(\epsilon)$ with positive $\epsilon$}
        \label{fig:dynamics_e_positive}
    \end{subfigure} 
\caption{The effect of the sign of $\epsilon$, as a function of dimension during training. We plot the percentage of correctly sorted input entries of the predictor ($y^*$) and the loss-augmented predictor ($y^*(\epsilon)$) as well as the loss. While the loss is similar for negative and positive $\epsilon$ when $d=10$, this changes for $d=100$ (Figure \ref{fig:dynamics_loss}). As such, the percentage of correctly sorted input entries of $y^*_w$ and $y^*(\epsilon)$ greatly differs when $d=100$ for different $\epsilon$. Importantly, when learning with $\epsilon > 0$ there are less than 40\% correct entries in $y^*_w$ (Figure \ref{fig:dynamics_e_positive}), while when learning with $\epsilon < 0$ there are at least $90\%$ correct entries in $y^*_w$ (Figure \ref{fig:dynamics_e_negative}).}
\label{fig:dynamics}
\vspace{-\baselineskip}
\end{figure*}

\citet{mena2018learning} have introduced two evaluation measures: the proportion of sequences where there was at least one error (Prop. Any Wrong), and the overall proportion of samples assigned to a wrong position (Prop. Wrong). They report the best achieved Prop. Any Wrong measure over an unspecified number of trials. To indicate robustness, we extend these measures to the following: Percentage of zero Prop. Any Wrong sequences, as well as Average and STD of Prop. Wrong, which are calculated over a number of training and testing repetitions.

We follow the Sorting Numbers experiment protocol of \citet{mena2018learning} and use the code released by the authors, to perform $20$ Sinkhorn iterations and $10$ different reconstruction for each batch sample. Also, the training set consists of $10$ random sequences of length $d$ and a test set that consists of a single sequence of the same length $d$. At test time, random noise is not added to the learned signal $\mu_{u,ij}(x,y_{ij})$. 
The results in Table \ref{tab:Permutation_all} show the measures calculated over $200$ repetitions of training and testing.

One can see that direct loss minimization performs better than Gumbel-Sinkhorn, and the gap is larger for longer sequences. One can also see that learning the variance of the noise improves the performance of the structured predictor in all three measures, when compared to direct loss minimization \citep{NIPS2010_4069}, in which the variance is set to zero, as well as to \cite{Lorberbom2018DirectOT}, in which the noise variance is set to one.

Running time comparison is given in Table \ref{tab:BiPartite_EpochTime}.

Our code may be found in \url{https://github.com/HeddaCohenIndelman/PerturbedStructuredPredictorsDirect}.

\begin{table}[ht!]
\captionof{table}{Bipartite Matching Evaluation Measures.\\
Results show Percentage of zero Prop. Any Wrong sequences of test set (i.e perfect sorting). Average and STD of Prop. Wrong in parenthesis. We show the effect of learning signal-to-noise ratio method `Direct Stochastic Learning' in comparison with `Direct $\bar{\sigma} = 0$' referring to direct loss minimization \citep{NIPS2010_4069}, which can be interpreted as setting the noise variance to zero, 'Direct $\bar{\sigma} = 1$' referring to \cite{Lorberbom2018DirectOT}, in which the noise variance is set to one, and `Gumbel-Sinkhorn' referring to \cite{mena2018learning}.
Training set setting of $10$ random sequences of length $d$ and a test set of a single sequence of length $d$. Results are calculated from 200 training and testing repetitions.}
\scriptsize	
\centering
\begin{tabular}{@{}lcccc@{}}
\toprule
d & \begin{tabular}[c]{@{}c@{}}Direct \\$\bar \sigma = 0$ \end{tabular} & 
\begin{tabular}[c]{@{}c@{}}Direct \\$\bar \sigma = 1$\end{tabular} &
\textbf{\begin{tabular}[c]{@{}c@{}}Direct Stochastic\\  Learning\end{tabular}} & \begin{tabular}[c]{@{}c@{}}Gumbel-Sinkhorn \end{tabular} \\ 
\midrule
5 & \begin{tabular}[c]{@{}c@{}}98.5\%\\ (0.6\%$\pm$4.9\%)\end{tabular}
& \begin{tabular}[c]{@{}c@{}}100\%\\ (0\%$\pm$0\%)\end{tabular}
& \textbf{\begin{tabular}[c]{@{}c@{}}100\%\\ (0\%$\pm$0\%)\end{tabular}} & \begin{tabular}[c]{@{}c@{}}100\%\\ (0\%$\pm$0\%)\end{tabular} \\ \midrule
10 & \begin{tabular}[c]{@{}c@{}}97\%\\ (0.6\%$\pm$3.4\%)\end{tabular} 
& \begin{tabular}[c]{@{}c@{}}100\%\\ (0\%$\pm$0\%)\end{tabular}
& \textbf{\begin{tabular}[c]{@{}c@{}}100\%\\ (0\%$\pm$0\%)\end{tabular}} & \begin{tabular}[c]{@{}c@{}}100\%\\ (0\%$\pm$0\%)\end{tabular} \\ \midrule
25 & \begin{tabular}[c]{@{}c@{}}89.5\%\\ (0.9\%$\pm$2.8\%)\end{tabular} & \begin{tabular}[c]{@{}c@{}}97.5\%\\ (0.3\%$\pm$1.7\%)\end{tabular} & \textbf{\begin{tabular}[c]{@{}c@{}}97.5\%\\ (0.3\%$\pm$1.6\%)\end{tabular}} & \begin{tabular}[c]{@{}c@{}}87.5\%\\ (1\%$\pm$3\%)\end{tabular} \\ \midrule
40 & \begin{tabular}[c]{@{}c@{}}84.5\%\\ (1.2\%$\pm$4.5\%)\end{tabular} & \begin{tabular}[c]{@{}c@{}}90.5\%\\ (0.6\%$\pm$2.2\%)\end{tabular}
& \textbf{\begin{tabular}[c]{@{}c@{}}91.6\%\\ (0.5\%$\pm$1.6\%)\end{tabular}} & \begin{tabular}[c]{@{}c@{}}83.5\%\\ (1\%$\pm$5\%)\end{tabular} \\ \midrule
60 & \begin{tabular}[c]{@{}c@{}}82\%\\ (0.9\%$\pm$2.6\%)\end{tabular} 
& \begin{tabular}[c]{@{}c@{}}80.0\%\\ (0.9\%$\pm$2.2\%)\end{tabular}
& \textbf{\begin{tabular}[c]{@{}c@{}}83.3\%\\ (0.7\%$\pm$1.8\%)\end{tabular}} & \begin{tabular}[c]{@{}c@{}}21\%\\ (5\%$\pm$9\%)\end{tabular} \\ \midrule
100 & \begin{tabular}[c]{@{}c@{}}74.9\%\\ (1.4\%$\pm$6.9\%)\end{tabular} 
& \begin{tabular}[c]{@{}c@{}}68.5\%\\ (1.2\%$\pm$2.4\%)\end{tabular}
& \textbf{\begin{tabular}[c]{@{}c@{}}76.8\%\\ (0.9\%$\pm$2.1\%)\end{tabular}} & \begin{tabular}[c]{@{}c@{}}0\%\\ (11.3\%$\pm$11.2\%)\end{tabular} \\ \bottomrule
\end{tabular}
\label{tab:Permutation_all}
\raggedbottom
\vspace{-\baselineskip}
\end{table}

\begin{table}[!htbp]
\caption{Comparison of average epoch time (seconds) of the bipartite matching experiment, per selected $d$}
\scriptsize
    \begin{tabular}{ccc}
    \toprule
    \multicolumn{1}{l}{d} & \textbf{Direct Stochastic Learning} & \textbf{Gumbel-Sinkhorn}  \\ \hline
    \multicolumn{1}{c}{10}  & 0.247 & 0.288  \\ \hline
    \multicolumn{1}{c}{40}  & 0.252 & 0.294 \\ \hline
    \multicolumn{1}{c}{100} & 0.304 & 0.306 \\  \bottomrule
    \end{tabular}
    \label{tab:BiPartite_EpochTime}
\vspace{-\baselineskip}
\end{table}

\subsection{k-Nearest Neighbors For Image Classification}

\begin{figure*}[!ht]
  \centering
    \includegraphics[width=0.5\textwidth]{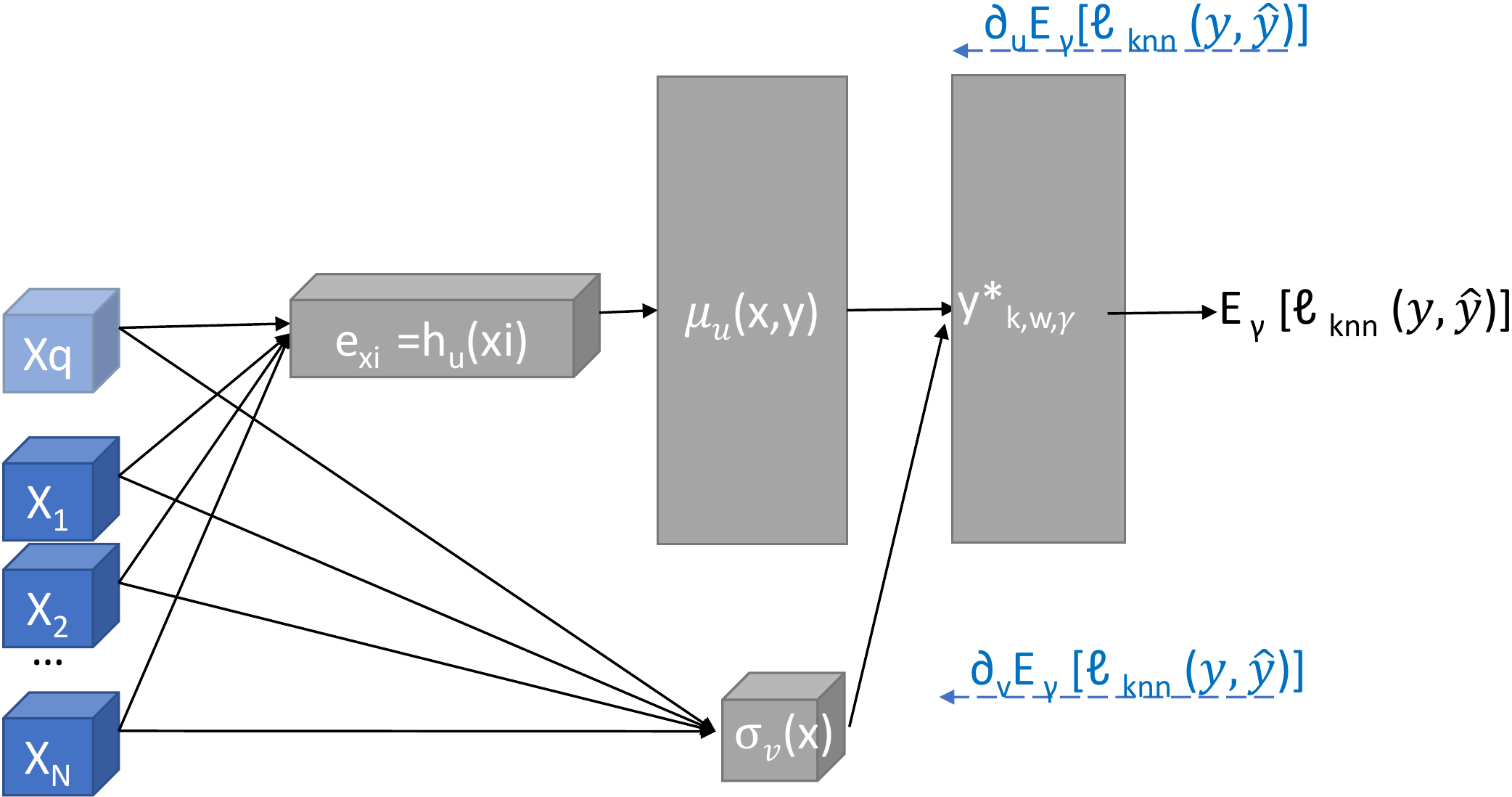}
    \caption{k-nn schematic architecture. The expectancy over Gumbel noise of the loss is derived w.r.t. the parameters $u$ of the signal and w.r.t. the parameters $v$ of the variance controller $\sigma$ directly (Equations \ref{eq:direct_mu},\ref{eq:direct_sigma} respectively). We deployed the same distance embedding networks as the ones deployed by \citet{grover2018stochastic} (Details in the supplementary material). Our prediction $y^*_{u,\gamma}$ yields the top-k images having the minimum Euclidean distance in embedding space (equivalently, the maximum negative Euclidean distance). The network $\sigma$ output is activated by a softplus activation in order to enforce a positive $\sigma$ value.}
    \label{knnDiagram}
    \vspace{-\baselineskip}
\end{figure*}

We follow the problem setting and architecture $\mu_{u,\alpha}(x,y_\alpha)$ and loss function $\ell(y,y^*)$ of  \citet{grover2018stochastic} for learning the $k$-nearest neighbors (kNN) classier. We replace the unimodal row stochastic matrix operation with our gradient step to directly minimize the distance to the closest $k$ candidates, see Figure \ref{knnDiagram}. In this experiment each training example $(x,y) \in S$ consists of an input vector $x = (x_1,...,x_n,x_q)$ of $n$ candidate images $x_1,...,x_n$ and a single query image $x_q$; its corresponding structured label $y = (y_1,...,y_n)$ . The structured label $y \in \{0,1\}^n$ points to the $k$ candidate images with minimum Euclidean distance to the query image, i.e., $\sum_{i=1}^n y_i = k$. Here we set $\alpha$ to be the index $i = 1,\dots, n$ that correspond to the top $k$ candidate images. $\mu_{u,i}(x,y_i)$ is the negative distance between the embedding $h_u(\cdot)$ of the $i$-th candidate image and that of the query image: $\mu_{u,i}(x,y_i) = -\|h_u(x_i) - h_u(x_q)\|$. 
Our prediction $y^*_{u,\gamma}$ yields the top-k images having the minimum Euclidean distance in embedding space (equivalently, the maximum negative Euclidean distance) $y^*_{u,k,\gamma} =  \arg \max_{\hat y \in Y} \{\mu_u(x,\hat y) +\gamma(\hat y)  \}$. Here, we set $Y$ to be the set of all structures $y \in \{0,1\}^n$ satisfying $\sum_{i=1}^n y_i = k$. The loss function is a linear function of its labels: $\ell(y,\hat y) = - \sum_{i=1}^n \|x_i - x_q\| y_i \hat y_i$. 
Our gradient update rule in Equation (\ref{eq:direct_mu}) replaces the unimodal row stochastic construction operator of \citet{grover2018stochastic}. We note that $y^*_{w,\gamma}$ and $y^*_{w,\gamma}(\epsilon)$ can be computed efficiently by extracting the top $k$ elements over $n$ elements. 

\begin{table}[!ht]
\captionof{table}{Test-Set Classification Average Accuracy, Per $k$.\\
We show the effect of our `Direct Stochastic Learning' method for learning signal-to-noise ratio in comparison with: `Direct $\bar{\sigma} = 0$' referring to direct loss minimization without random noise \citep{NIPS2010_4069}, `Direct $\bar{\sigma} = 1$' referring to \citet{Lorberbom2018DirectOT}, in which the noise variance is set to one, `NeuralSort' referring to \citet{grover2018stochastic}, and `RelaxSubSample' referring to \citet{DBLP:conf/ijcai/XieE19} who quote results for $k=5$ only.}
\scriptsize	
\begin{subtable}[t]{\columnwidth}
\begin{tabular}{lcccc}
\toprule
&\multicolumn{4}{c}{MNIST}  \\ \hline
& k=1 & k=3 & k=5 & k=9 \\ \hline
Direct $\bar \sigma=0 $ & 99.1\% & 99.2\% & 99.3\% & 99.2\% \\ \hline
Direct $\bar \sigma=1 $ & 16\% & 53.84\% & 14.25\% & 41.53\% \\ \hline
\textbf{\begin{tabular}[c]{@{}l@{}}Direct Stochastic Learning\end{tabular}} & \textbf{99.34\%} & 99.4\% & \textbf{99.4\%} & 99.34\% \\ \hline
\begin{tabular}[c]{@{}l@{}}NeuralSort deterministic \end{tabular} & 99.2\% & \textbf{99.5\%} & 99.3\% & 99.3\%  \\ \hline
\begin{tabular}[c]{@{}l@{}}NeuralSort stochastic \end{tabular} & 99.1\% & 99.3\% & \textbf{99.4\%} & \textbf{99.4\%}\\ \hline
\begin{tabular}[c]{@{}l@{}}RelaxSubSample\end{tabular} &  &  & 99.3\% & \\ 
\bottomrule
\end{tabular}
\end{subtable}

\vspace{0.1cm}

\begin{subtable}[t]{\columnwidth}
\begin{tabular}{lcccc}
&\multicolumn{4}{c}{Fashion-MNIST} \\ \hline
& k=1 & k=3 & k=5 & k=9 \\ \hline
Direct $\bar \sigma=0 $  & 89.8\% & 93.2\% & 93.5\% & \textbf{93.7\%} \\ \hline
Direct $\bar \sigma=1 $ & 92.5\% & \textbf{93.4\%} & 93.3\% & 93.2\% \\ \hline
\textbf{\begin{tabular}[c]{@{}l@{}}Direct Stochastic  Learning\end{tabular}} & \textbf{92.6\%} & 93.3\% & \textbf{94\%} & \textbf{93.7\%} \\ \hline
\begin{tabular}[c]{@{}l@{}}NeuralSort deterministic \end{tabular} & \textbf{92.6\%} & 93.2\% & 93.5\% & 93\% \\ \hline
\begin{tabular}[c]{@{}l@{}}NeuralSort stochastic \end{tabular} & 92.2\% & 93.1\% & 93.3\% & 93.4\% \\ \hline
\begin{tabular}[c]{@{}l@{}}RelaxSubSample\end{tabular} &  &  & 93.6\% & \\ 
\bottomrule
\end{tabular}
\end{subtable}

\vspace{0.1cm}

\begin{subtable}[t]{\columnwidth}
\begin{tabular}{lcccc}
&\multicolumn{4}{c}{CIFAR-10} \\ \hline
& k=1 & k=3 & k=5 & k=9 \\ \hline
Direct $\bar \sigma=0 $  & 24.9\% & 27\% & 39.6\% & 39.9\% \\ \hline
Direct $\bar \sigma=1 $ & 23.1\% & 89.95\% & 90.85\% & 91.6\% \\ \hline
\textbf{\begin{tabular}[c]{@{}l@{}}Direct Stochastic Learning\end{tabular}} & 29.6\% & \textbf{90.7\%} & \textbf{91.25\%} & \textbf{91.7\%} \\ \hline
\begin{tabular}[c]{@{}l@{}}NeuralSort deterministic \end{tabular} & \textbf{88.7\%} & 90.0\% & 90.2\% & 90.7\%  \\ \hline
\begin{tabular}[c]{@{}l@{}}NeuralSort stochastic \end{tabular} & 85.1\% & 87.8\% & 88.0\% & 89.5\%\\ \hline
\begin{tabular}[c]{@{}l@{}}RelaxSubSample \end{tabular} &  &  & 90.1\% & \\ 
\bottomrule
\end{tabular}
\end{subtable}
\label{tab:knn_vertical}
\end{table}

We report the classification accuracies on the standard test sets in Table \ref{tab:knn_vertical}. For MNIST and Fashion-MNIST, our method matched or outperformed `NeuralSort' \citep{grover2018stochastic} and `RelaxSubSample' \citep{DBLP:conf/ijcai/XieE19}, in all except $k=3, 9$ in MNIST. For CIFAR-10, our method outperformed `NeuralSort' and `RelaxSubSample', in all except k=1, for which disappointing results are attained by all direct loss based methods.
We note that `Direct $\bar{\sigma} = 0$' seems to suffer from very low average accuracy on CIFAR-10 dataset. Additionally,  `Direct $\bar{\sigma} = 1$' suffer from very low average accuracy on MNIST dataset. It is evident that our method stabilizes the performance on all datasets.

Running time comparison for $k=3$ is given in Table \ref{tab:knn_avgbatch}. Performance is robust to $k$ in both methods. 
\begin{table}[!htbp]
\vspace{-0.2cm}
\caption{Comparison of average epoch running time (seconds) of the k-nn experiment. Results are for $k=3$.}
\scriptsize
\begin{tabular}{ccc}
    \hline
     & \textbf{Direct Stochastic Learning} & \textbf{NeuralSort Stochastic} \\ \hline
    MNIST & 28.3 & 14.4 \\ \hline
    Fashion-MNIST & 198.6 & 328.6 \\ \hline
    CIFAR-10 & 220. & 337.6 \\ \hline
\end{tabular}
\label{tab:knn_avgbatch}
\end{table}

\section{Discussion And Future Work}
In this work, we learn the mean and the variance of structured predictors, while directly minimizing their loss. Our work extends direct loss minimization as it explicitly adds random perturbation to the prediction process to better control the relation between data instance and its exponentially many possible structures. Our work also extends direct optimization through the $\arg \max$ in generative learning as it adds a variance term to better balance the learned signal with the perturbed noise. The experiments validate the benefit of our approach. 

The structured distributions that are implied from our method are different than the standard Gibbs distribution, when the localized score functions are over subsets of variables. The exact relation between these distributions and the role of the Gumbel distribution law in the structured setting is an open problem. There are also optimization-related questions that arise from our work, such as exploring the role of $\epsilon$ and its impact on the convergence of the algorithm.

\bibliographystyle{abbrvnat}
\setcitestyle{authoryear,open={(},close={)}}
\bibliography{references}

\appendix
\setcounter{section}{0}

\onecolumn

\icmltitle{Learning Randomly Perturbed Structured Predictors for Direct Loss Minimization
Supplementary Material}

\section{Connecting variance of Gumbel random variables and temperature of Gibbs models}\label{Gibbs_temperature_proof}
We focus on the Gumbel distribution with zero mean, which is described by its a double exponential cumulative distribution function 
\begin{equation} \label{eq:g_CFD}
    G(t) = P(\gamma(y) \leq t) = e^{-e^{-(t + c)}}
\end{equation}

where $c \approx 0.5772$ is the Euler-Mascheroni constant.

We will show that then one obtains the following identity:
\begin{equation}
    e^{\frac{\mu_w(x,y)}{\sigma(x)}}  = P_{\gamma \sim g}[y^* = y]
\end{equation}
, when 
\begin{equation}
    y^* = \arg \max_{\hat y} \{ \mu_w (x,\hat y) + \sigma(x) \gamma(\hat y)\}.
\end{equation}

Let us define:
$\hat{Z}(\mu,\sigma) = \sum_{y \in Y}e^{\frac{\mu(x,y)}{\sigma(x)}}$ 
and
$Z(\mu) = \sum_{y \in Y}e^{\mu(x,y)}$ 

\begin{theorem} \label{theorem_logpartitionsigma} 
 Let $\gamma = \{\gamma(y) : y \in Y \}$ be a collection of i.i.d. Gumbel random
variables with cumulative distribution function (\ref{eq:g_CFD}). 
Then, the random variable 
$\max_{y \in Y} \{\frac{\mu(x,y)} {\sigma(x)} + \gamma(y)\}$ is distributed according to the Gumbel distribution whose mean is the log-partition
function $log \hat{Z}(\mu,\sigma)$.
\end{theorem}

\begin{proof}
We denote by $F(t) = P(\gamma(y) \leq t)$ the cumulative
distribution function of $\gamma(y)$.

The independence of $\gamma(y)$ across
$y\in Y$ implies that:
\begin{equation*}
\begin{split}
P_{\gamma}(\max_{y\in Y} \{\frac{\mu(x,y)} {\sigma(x)} + \gamma(y)\} \leq t ) 
      &= P_{\gamma}(\forall_{y\in Y}\{\frac{\mu(x,y)} {\sigma(x)} + \gamma(y) \} \leq t)\\ 
      &= P_{\gamma}(\forall_{y\in Y}\{\gamma(y)  \} \leq t - \frac{\mu(x,y)} {\sigma(x)} )\\
      &= \prod_{y\in Y}F(t - \frac{\mu(x,y)} {\sigma(x)} )\\
\end{split}   
\end{equation*}

The Gumbel, Frechet, and Weibull distributions, used in extremal statistics, are max-stable distributions: the product 
$\prod_{y \in Y}F(t - \frac{\mu(x,y)} {\sigma(x)} )$ can be described in terms of F(·) itself. Under the said setting,
the double exponential form of the Gumbel distribution yields the result:
\begin{equation*}
\begin{split}
\prod_{y \in Y}F(t - \frac{\mu(x,y)} {\sigma(x)}) &= e^{- \sum_{y\in Y}e^{- (t - \frac{\mu(x,y)} {\sigma(x)} + c )}} \\
 &= e^{- e^{- (t + c - \log \hat{Z}(\mu,\sigma) )}} \\
 &= F(t-\log \hat{Z}(\mu,\sigma))
\end{split}   
\end{equation*}  
\end{proof}

\begin{corollary} \label{main_corollary}
 Let $\gamma = \{\gamma(y) : y \in Y \}$ be a collection of i.i.d. Gumbel random
variables with cumulative distribution function (\ref{eq:g_CFD}). Then, for all $\hat{y}$:
\begin{equation*}
\frac{e^{\frac{\mu(x,y)} {\sigma(x)}}}{\hat{Z}(\mu,\sigma)} = P_{\gamma}(\hat{y} = \arg \max_{y \in Y} \{\mu(x,y) + \sigma(x) \gamma(y)\})
\end{equation*} 
\end{corollary}
\begin{proof}
For Gumbel random
variables with cumulative distribution function (\ref{eq:g_CFD}) it holds:
\begin{equation} 
\label{GumbelProperty}
G'(t)  = e^{-t}G(t) 
 = e^{-t}e^{- e^{-t}} 
 = e^{-t-e^{-t}}
 = e^{-(t+e^{-t})} 
 = g(t) 
\end{equation}
g(t) is the probability density function of the standard Gumbel distribution.

We note that:
\begin{equation*}
\begin{split}
P_{\gamma}(\max_{y \in Y} \{\mu(x,y) + \sigma(x) \gamma(y)\})     & = \frac{\sigma(x)}{\sigma(x)}P_{\gamma}(\max_{y \in Y}         \{\mu(x,y) + \sigma(x) \gamma(y)\})\\
    & = \sigma(x) P_{\gamma}(\max_{y \in Y} \{\frac{\mu(x,y)} {\sigma(x)} +  \gamma(y)\})
\end{split}
\end{equation*}

From Theorem \ref{theorem_logpartitionsigma}, we have
$\mathbb{E}_{\gamma}[\max_{y \in Y} \{\frac{\mu(x,y)} {\sigma(x)} + \gamma(y)\}] = \log \hat{Z}(\mu,\sigma)  $

Putting it together we have that:
$\mathbb{E}_{\gamma}(\max_{y \in Y} \{\mu(x,y) + \sigma(x) \gamma(y)\}) = \sigma(x) \log \hat{Z}(\mu,\sigma)$. 
We can derive w.r.t. some $\mu'(x,y)$.

We note that by differentiating the right hand side we get:
\begin{equation*}
\begin{split}
\frac{\partial (\sigma(x) \log \hat{Z}(\mu,\sigma))}{\partial\mu '(x,y)} 
&= \frac{e^{\frac{\mu '(x,y)}{\sigma(x)}}}{ \hat{Z}(\mu,\sigma)}
\end{split}
\end{equation*}

Differentiate the left hand side:
First, we can differentiate under the integral sign:
\begin{equation*}
\frac{\partial}{\partial\mu '(x,y)} \int_{\rm I\!R^{|Y|}} \max_{y \in Y} \{\mu(x,y) + \sigma(x) \gamma(y)\} d\gamma
=  \int_{\rm I\!R^{|Y|}}\frac{\partial}{\partial\mu '(x,y)} \max_{y \in Y} \{\mu(x,y) + \sigma(x) \gamma(y)\} d\gamma
\end{equation*}  
We can write a subgradient of the max-function using an
indicator function (an application of Danskin’s Theorem):
\begin{equation*}
\frac{\partial}{\partial\mu^{\prime}(x,y)}max_{y \in Y} \{\mu(x,y) + \sigma(x) \gamma(y)\} = \mathds{1}(\hat{y} = \arg \max_{y \in Y} \{\mu(x,y) + \sigma(x) \gamma(y)\})  
\end{equation*}  
The corollary then follows by applying the expectation to both
sides of the last equation.
\end{proof}

An alternative proof of the preceding corollary can also be made.
We begin by noting that:
$P_{\gamma}(\hat{y} = \arg \max_{y \in Y} \{\mu(x,y) + \sigma(x) \gamma(y)\} ) 
= P_{\gamma}(\hat{y} = \arg \max_{y \in Y} \{\frac{\mu(x,y)}{\sigma(x)} + \gamma(y)\} )$.
Then,
\begin{equation*}
\begin{split}
P_{\gamma}(\hat{y} = \arg \max_{y \in Y} \{\mu(x,y) + \sigma(x) \gamma(y)\} ) 
    &= \int_{t} G' (t - \frac{\mu(x,y)} {\sigma(x)}) \prod_{\hat{y} \neq y} G(t-\frac{\mu(x,\hat{y})} {\sigma(x)})dt\\
    \star &= \int_{t} e^{-(t- \frac{\mu(x,y)}{\sigma(x)})}G(t-\frac{\mu(x,y)}{\sigma(x)})\prod_{\hat{y} \neq y} G(t-\frac{\mu(x,\hat{y})}{\sigma(x)})dt\\
    &= e^{\frac{\mu(x,y)}{\sigma(x)}} \int_{t} e^{-t}\prod_{\hat{y}} G(t-\frac{\mu(x, \hat{y})}{\sigma(x)})dt\\
    &= e^{\frac{\mu(x,y)}{\sigma(x)}}*constant
\end{split}
\end{equation*}

Therefore the probability that $\hat{y}$ maximizes $\mu(x,y) + \sigma(x) \gamma(y)$ is proportional to
$ e^{\frac{\mu(x,y)}{\sigma(x)}}$.

$\star$ is due to the probability density function of the Gumbel distribution as shown in (\ref{GumbelProperty}).

\section{Proof of Corollary 2}

Recall that we defined the prediction $y^*_{w,\gamma}=$
\begin{equation}
     \arg \max_{\hat y} \Big\{ \sum_{\alpha \in {\cal A}} \mu_{u,\alpha}(x,\hat y_\alpha) + \sum_{i=1}^n \sigma_v(x) \gamma_i(\hat y_i) \Big\} \label{eq:y_dsl}
\end{equation}
The loss-perturbed prediction $y^*_{w,\gamma}(\epsilon)=$
\begin{equation}
     \arg \max_{\hat y} \Big\{ \sum_{\alpha \in {\cal A}} \mu_{u,\alpha}(x,\hat y_\alpha) + \sum_{i=1}^n \sigma_v(x) \gamma_i(\hat y_i) + \epsilon \ell(y, \hat y) \Big\} \label{eq:y_dsl_epsilon}
\end{equation}

$w=(u,v)$ are the learned parameters. 

Our aim is to prove the following gradient steps:
 $\frac{\partial}{\partial u} \mathbb{E}_{\gamma}[\ell(y, y^*_{w, \gamma})] =$ 
\[
\lim_{\epsilon\to 0} \frac{1}{\epsilon} \mathbb{E}_{\gamma}\Big[ \sum_{\alpha \in {\cal A}} \left( \nabla \mu_{u,\alpha}(x,y^*_\alpha(\epsilon)) - \nabla \mu_{u,\alpha}(x,y^*_\alpha) \right)  \Big]  \]
and  $\frac{\partial}{\partial v} \mathbb{E}_{\gamma}[\ell(y, y^*_{w, \gamma})] =$ 
\[ 
\lim_{\epsilon\to 0} \frac{1}{\epsilon} \mathbb{E}_{\gamma} \Big[\sum_{i=1}^n  \nabla \sigma_{v}(x) \Big(\gamma_i (y^*_i(\epsilon)) - \gamma_i(y^*_i)\Big)\Big] . 
\]

When we use the shorthand notation $y^*_\alpha = y^*_{w,\gamma, \alpha}$ and $y_i^* = y^*_{i,w,\gamma}$ and similarly $y^*_\alpha(\epsilon) = y^*_{w,\gamma, \alpha}(\epsilon)$ and $y_i^*(\epsilon) = y^*_{i,w,\gamma}(\epsilon)$ and recall that $w$ refers to $u$ and $v$.

The main challenge is to show that $G(u,v,\epsilon)$, as defined in Equation (\ref{G_pred_dsl}), is differentiable, i.e., there exists a vector $\sum_\alpha \nabla_u \mu_{u, \alpha }(x, y_{u,v,\gamma,\alpha}^*(\epsilon))$ such that for any direction $z$, its corresponding directional derivative $\lim_{h \rightarrow 0} \frac{G(u + hz, v, \epsilon) - G(u,v,\epsilon)}{h}$ equals $\E_{\gamma \sim {\cal G}} [\sum_\alpha \nabla_u \mu_{u, \alpha }(x, y_{u,v,\gamma,\alpha}^*(\epsilon))^\top z]$.

Similarly, we will show that there exists a vector $\sum_{i=1}^n \nabla_v \sigma_{v}(x)\gamma_i(y_{i,v,u,\gamma}^*(\epsilon))$ such that for any direction $z$, 
its corresponding directional derivative

$\lim_{h \rightarrow 0} \frac{G(u, v+ hz, \epsilon) - G(u,v,\epsilon)}{h}$ equals $\E_{\gamma \sim {\cal G}} [
\sum_{i=1}^n \nabla_v \sigma_{v}(x)\gamma_i(y_{i,u,v,\gamma}^*(\epsilon))^\top z]$. 

This challenge is addressed in Theorem \ref{main_theorem_DSL}, which also utilizes Lemma \ref{lemma:maxfn_dsl}. This lemma relies of the discrete nature of the label space, ensuring that the optimal label does not change in the vicinity of $y^*_{u,v, \gamma}(\epsilon)$.
The proof concludes by the Hessian of $G(u,v,\epsilon)$ symmetric entries in Corollary \ref{corollary_DSL_appendix}. 

\begin{lemma}\label{lemma:maxfn_dsl}
Assume $\{\mu_{u, \alpha}(x,y_\alpha)\}$ is a set of continuous functions of $u$ for  $\alpha \in \cal A$ and assume $\sigma_v(x)$ is a smooth function of $v$. Let $\gamma_i(y_i)$ be i.i.d. random variables with a smooth probability density function $\cal G$. Assume that the loss-perturbed maximal arguments $y^*_{u + \frac{1}{n} z,v, \gamma}(\epsilon)$ and $y^*_{u ,v +\frac{1}{n} z, \gamma}(\epsilon)$, as defined in Equation (\ref{eq:y_dsl_epsilon}), are unique for any $z$ and $n$. Then, there exists $n_0$ such that for $n \ge n_0$ there holds 
\begin{equation}
    y^*_{u + \frac{1}{n} z,v, \gamma}(\epsilon) = y_{u,v,\gamma}^*(\epsilon) \label{maxfn_dsl_u}
\end{equation}
and there exists $n_1$ such that for $n \ge n_1$ there holds 
\begin{equation}
    y^*_{u ,v +\frac{1}{n} z, \gamma}(\epsilon)=y_{u,v,\gamma}^*(\epsilon) \label{maxfn_dsl_v}.
\end{equation}
\end{lemma}

\begin{proof}
We will first prove Equation (\ref{maxfn_dsl_u}). Let $f_n(\hat y) = \sum_\alpha \mu_{u + \frac{1}{n} z, \alpha}(x,\hat y_\alpha) + \sum_{i=1}^n \sigma_v(x) \gamma_i(\hat y_i) + \epsilon \ell(y, \hat y)$ so that $y^*_{u + \frac{1}{n} z,v, \gamma}(\epsilon) = \arg \max_{\hat y}  f_n(\hat y)$. Also, let $f_\infty(\hat y) = \sum _\alpha \mu_{u, \alpha}(x,\hat y_\alpha) + \sum_{i=1}^n \sigma_v(x) \gamma_i(\hat y_i) + \epsilon \ell(y, \hat y)$ 
so that $y^*_{u, v, \gamma}(\epsilon) = \arg \max_{\hat y} f_\infty(\hat y)$. Since $f_n(\hat y)$ is a continuous function then $\max_{\hat y} f_n(\hat y)$ is also a continuous function and $\lim_{n \rightarrow \infty} \max_{\hat y} f_n(\hat y) = \max_{\hat y} f_\infty(\hat y)$. Since $\max_{\hat y}  f_n(\hat y) = f_n(y^*_{u + \frac{1}{n} z, v, \gamma}(\epsilon))$ is arbitrarily close to $\max_{\hat y} f_\infty(\hat y) = f_\infty(y^*_{u,v, \gamma}(\epsilon))$, and $y^*_{u,v,\gamma}(\epsilon), y^*_{u + \frac{1}{n} z, v, \gamma}(\epsilon)$ are unique then for any $n \ge n_0$ these two arguments must be the same, otherwise there is a $\delta > 0$ for which $| f_\infty(y^*_{u,v, \gamma}(\epsilon)) - f_n(y^*_{u + \frac{1}{n} z, v ,\gamma}(\epsilon)) | \ge \delta$. 

To prove Equation (\ref{maxfn_dsl_v}), one can define 

$f'_n(\hat y) = \sum_\alpha \mu_{u, \alpha}(x,\hat y_\alpha) + \sum_{i=1}^n \sigma_{v+ \frac{1}{n} z}(x) \gamma_i(\hat y_i) + \epsilon \ell(y, \hat y)$ and follow the same steps to show that for any $n \ge n_1$ it holds that $f'_\infty(y^*_{u,v, \gamma}(\epsilon))$ and  $f'_n(y^*_{u, v+ \frac{1}{n} z ,\gamma}(\epsilon))$ are arbitrarily close and since $y^*_{u ,v +\frac{1}{n} z, \gamma}(\epsilon)$ and $y_{u,v,\gamma}^*(\epsilon)$ are unique then $y^*_{u ,v +\frac{1}{n} z, \gamma}(\epsilon)=y_{u,v,\gamma}^*(\epsilon)$. 
\end{proof}

\begin{theorem}\label{main_theorem_DSL}
Assume that $E_{\gamma \sim {\cal G}} \|\nabla_u \mu_{u,\alpha}(x,y_\alpha)\| \le \infty$, and that $E_{\gamma \sim \cal G} \|\nabla_v \sigma_v(x)\| \le \infty$.
Define the prediction generating function $G(u,v,\epsilon) =$
\begin{equation}\label{G_pred_dsl}
\mathbb{E}_{\gamma \sim {\cal G}} \Big[   \max_{\hat y \in Y} \Big\{ \sum_{\alpha \in {\cal A}} \mu_{u,\alpha}(x,\hat y_\alpha) + \sum_{i=1}^n \sigma_v(x) \gamma_i(\hat y_i) + \epsilon \ell(y, \hat y)  \Big\}  \Big].  
\end{equation}
If the conditions of Lemma \ref{lemma:maxfn_dsl} hold then $G(u,v,\epsilon)$ as defined in Equation (\ref{G_pred_dsl}) is differentiable and 
\[
    \frac{\partial G(u,v,\epsilon)}{\partial u}  &=& \E_\gamma \Big[\sum_\alpha \nabla_u \mu_{u, \alpha}(x, y^*_{\alpha}(\epsilon)) \Big] \label{dGdu}\\
   \frac{\partial G(u,v,\epsilon)}{\partial v}  &=& \mathbb{E}_{\gamma} \Big[\sum_{i=1}^n  \nabla \sigma_{v}(x) \gamma_i (y^*_i(\epsilon)) \Big] \label{dGdv}
\]
\end{theorem}

\begin{proof}
We will first prove Equation (\ref{dGdu}). Let $f_n(\hat y) = \sum_\alpha \mu_{u + \frac{1}{n} z, \alpha}(x,\hat y_\alpha) + \sum_{i=1}^n \sigma_v(x) \gamma_i(\hat y_i) + \epsilon \ell(y, \hat y)$ 
as in Lemma \ref{lemma:maxfn_dsl}.

The proof builds a sequence of functions $\{g_n(z)\}_{n=1}^\infty$ that satisfies   
\[
\lim_{h \rightarrow 0} \frac{G(u + hz, v, \epsilon) - G(u,v,\epsilon)}{h} = \lim_{n \rightarrow \infty} \E_{\gamma \sim {\cal G}} [g_n(z)] \label{eq:grad1_dsl} \\
\E_{\gamma \sim {\cal G}} [\lim_{n \rightarrow \infty}  g_n(z)] =\E_{\gamma \sim {\cal G}} [\sum_\alpha \nabla_u \mu_{u, \alpha}(x, y_{u,v,\gamma, \alpha}^*(\epsilon))^\top z]. \label{eq:grad2_dsl}
\]
The functions $g_n(z)$ correspond to the loss perturbed prediction $y_{u,v,\gamma}^*(\epsilon)$ through the quantity $\sum_\alpha \mu_{u + \frac{1}{n} z, \alpha}(x, \hat y_\alpha) + \sum_{i=1}^n \sigma_v(x) \gamma_i(\hat y_i) + \epsilon \ell(y, \hat y)$. The key idea we are exploiting is that there exists $n_0$ such that for any $n \ge n_0$ the maximal argument $y_{u + \frac{1}{n} z,v,\gamma}^*(\epsilon)$  does not change. 

Thus, let 
\[
    g_n(z) \triangleq \frac{\max_{\hat y \in Y} f_n( \hat y) - \max_{\hat y \in Y} f_\infty(\hat y)}{1/n} \label{eq:gn_dsl}
\]
We apply the dominated convergence theorem on $g_n(z)$, so that $ \lim_{n \rightarrow \infty} \E_{\gamma \sim {\cal G}} [g_n(z)] = \E_{\gamma \sim {\cal G}} [\lim_{n \rightarrow \infty}  g_n(z)]$ in order to prove Equations (\ref{eq:grad1_dsl},\ref{eq:grad2_dsl}). 
We note that we may apply the dominated convergence theorem, since the conditions $E_{\gamma \sim {\cal G}} \|\nabla_u \mu_{u,\alpha}(x,y_\alpha)\| \le \infty$, and $E_{\gamma \sim \cal G} \|\nabla_v \sigma_v(x)\| \le \infty$ imply that the expected value of $g_n$ is finite (We recall that $f_n$ is a measurable function, and note that since $\hat y \in Y$ is an element from a discrete set $Y$, then $g_n$ is also a measurable function.). 

From Lemma \ref{lemma:maxfn_dsl}, the terms $ \ell(y, y^*)$ and $\sum_{i=1}^n \sigma_v(x) \gamma_i( y^*_i)$ are identical in both $\max_{\hat y \in Y} f_n(\hat y) $ and $\max_{\hat y \in Y} f_\infty(\hat y)$. Therefore, they cancel out when computing the difference $\max_{\hat y \in Y} f_n(\hat y)  - \max_{\hat y \in Y} f_\infty(\hat y) $. Then, for $n \ge n_0$:
\begin{equation*}
\max_{\hat y \in Y} f_n(\hat y) - \max_{\hat y \in Y} f_\infty(\hat y) = \sum_\alpha \mu_{u + \frac{1}{n} z,\alpha}(x, y_{\alpha}^*(\epsilon))
-\sum_\alpha \mu_{u,\alpha}(x, y_{\alpha}^*(\epsilon))
\end{equation*}
and Equation (\ref{eq:gn_dsl}) becomes:
\begin{equation}
g_n(u) = \frac{ \sum_\alpha \mu_{u + \frac{1}{n} z,\alpha}(x, y_{\alpha}^*(\epsilon)) - \sum_\alpha \mu_{u,\alpha}(x, y_{\alpha}^*(\epsilon))}{1/n} \label{gn_max_f_dsl}.
\end{equation}
Since $\{\mu_{u, \alpha}(x,y_\alpha)\}$ is a set of continuous functions of $u$ , then $\lim_{n \rightarrow \infty} g_n(z) $ is composed of the derivatives of $\mu_{u,\alpha}(x, y_{\alpha}^*(\epsilon))$ in direction $z$, namely, $\lim_{n \rightarrow \infty} g_n(z)= \sum_\alpha \nabla_u \mu_{u,\alpha}(x, y_{\alpha}^*(\epsilon))^\top z$. 

We now turn to prove Equation (\ref{dGdv}). 

Let $f'_n(\hat y) = \sum_\alpha \mu_{u, \alpha}(x,\hat y_\alpha) + \sum_{i=1}^n \sigma_{v+ \frac{1}{n} z}(x) \gamma_i(\hat y_i) + \epsilon \ell(y, \hat y)$ 
as in Lemma \ref{lemma:maxfn_dsl}.

The proof builds a sequence of functions $\{g'_n(z)\}_{n=1}^\infty$ that satisfies   

\[
\lim_{h \rightarrow 0} \frac{G(u, v+ hz, \epsilon) - G(u,v,\epsilon)}{h} = \lim_{n \rightarrow \infty} \E_{\gamma \sim {\cal G}} [g'_n(z)] \label{eq:grad_v_1_dsl} \\
\E_{\gamma \sim {\cal G}} [\lim_{n \rightarrow \infty}  g'_n(z)] =\E_{\gamma \sim {\cal G}} [\sum_{i=1}^n \nabla_v \sigma_{v}(x)\gamma_i(y_{i,u,v,\gamma}^*(\epsilon))^\top z]. \label{eq:grad_v_2_dsl}
\]
The functions $g'_n(z)$ correspond to the loss perturbed prediction $y_{u,v,\gamma}^*(\epsilon)$ through the quantity $\sum_\alpha \mu_{u, \alpha}(x, \hat y_\alpha) + \sum_{i=1}^n \sigma_{v + \frac{1}{n} z}(x) \gamma_i(\hat y_i) + \epsilon \ell(y, \hat y)$. 

As before, we are exploiting the fact that there exists $n_1$ such that for any $n \ge n_1$ the maximal argument $y_{u ,v+ \frac{1}{n} z,\gamma}^*(\epsilon)$  does not change. 

Thus, let 
\[
    g'_n(z) \triangleq \frac{\max_{\hat y \in Y} f'_n( \hat y) - \max_{\hat y \in Y} f'_\infty(\hat y)}{1/n} \label{eq:gn_v_dsl}
\]

We apply the dominated convergence theorem on $g'_n(z)$, so that $ \lim_{n \rightarrow \infty} \E_{\gamma \sim {\cal G}} [g'_n(z)] = \E_{\gamma \sim {\cal G}} [\lim_{n \rightarrow \infty}  g'_n(z)]$ in order to prove Equations (\ref{eq:grad_v_1_dsl},\ref{eq:grad_v_2_dsl}), with the same justification as before for the expected value of $g'_n(z)$ being finite. 

From Lemma \ref{lemma:maxfn_dsl}, the terms $ \ell(y, y^*)$ and $\sum_\alpha \mu_{u, \alpha}(x, y^*_\alpha)$ are identical in both $\max_{\hat y \in Y} f'_n(\hat y) $ and $\max_{\hat y \in Y} f'_\infty(\hat y)$. 

Therefore, they cancel out when computing the difference $\max_{\hat y \in Y} f'_n(\hat y)  - \max_{\hat y \in Y} f'_\infty(\hat y) $. Then, for $n \ge n_1$ 
Equation (\ref{eq:gn_v_dsl}) becomes:
\begin{equation}
g'_n(u) = \frac{\sum_{i=1}^n \sigma_{v+ \frac{1}{n} z}(x) \gamma_i(y^*_i(\epsilon))- \sum_{i=1}^n \sigma_{v}(x) \gamma_i(y^*_i(\epsilon))}{1/n} \label{gn_max_f'_dsl}.
\end{equation}

Since $\sigma_v(x)$ is a smooth function of $v$, then $\lim_{n \rightarrow \infty} g'_n(z) $ is composed of the derivatives of $\sigma_{v}(x) \gamma_i(y^*_i(\epsilon))$ in direction $z$, namely, $\lim_{n \rightarrow \infty} g'_n(z)= \sum_{i=1}^n \nabla_v \sigma_{v}(x) \gamma_i(y^*_i(\epsilon))^\top z$. 

\end{proof}

\begin{corollary}\label{corollary_DSL_appendix}
Under the conditions of Theorem \ref{main_theorem_DSL},  $G(u,v,\epsilon)$  as defined in Equation (\ref{G_pred_dsl}),  is a smooth function and $\frac{\partial}{\partial u} \mathbb{E}_{\gamma}[\ell(y, y^*_{w, \gamma})] =$ 
\[
\lim_{\epsilon\to 0} \frac{1}{\epsilon} \mathbb{E}_{\gamma}\Big[ \sum_{\alpha \in {\cal A}} \left( \nabla \mu_{u,\alpha}(x,y^*_\alpha(\epsilon)) - \nabla \mu_{u,\alpha}(x,y^*_\alpha) \right)  \Big] \label{eq:dsl_mu} \]
and  $\frac{\partial}{\partial v} \mathbb{E}_{\gamma}[\ell(y, y^*_{w, \gamma})] =$ 
\[ 
\lim_{\epsilon\to 0} \frac{1}{\epsilon} \mathbb{E}_{\gamma} \Big[\sum_{i=1}^n  \nabla \sigma_{v}(x) \Big(\gamma_i (y^*_i(\epsilon)) - \gamma_i(y^*_i)\Big)\Big] .  \label{eq:dsl_sigma}
\]
\end{corollary}

\begin{proof}
We will first prove Equation (\ref{eq:dsl_mu}). Since Theorem \ref{main_theorem_DSL} holds for every direction $z$:
\begin{equation*}
\frac{\partial G(u,v,\epsilon)}{\partial u} = \E_\gamma \Big[\sum_\alpha \nabla_u \mu_{u, \alpha}(x, y^*_\alpha(\epsilon)) \Big].
\end{equation*}
Adding a derivative with respect to $\epsilon$ we get: 
\begin{flalign*}
\frac{\partial}{\partial \epsilon} \frac{\partial G(u,v,0)}{\partial u} = \\
\lim_{\epsilon \rightarrow 0} &\frac{1}{\epsilon} \E_\gamma \Big[\sum_\alpha \nabla_u \mu_{u, \alpha}(x, y^*_\alpha(\epsilon)) -  \sum_\alpha \nabla_u \mu_{u, \alpha}(x, y^*_\alpha)  \Big]
\end{flalign*}

The proof follows by showing that the gradient computation is apparent in the Hessian, namely Equation (\ref{eq:dsl_mu}) is attained by the identity 
$\frac{\partial}{\partial u} \frac{\partial G(u,v,0)}{\partial \epsilon} = \frac{\partial}{\partial \epsilon} \frac{\partial G(u,v,0)}{\partial u}$.

Now we turn to show that $\frac{\partial}{\partial u} \frac{\partial G(u,v,0)}{\partial \epsilon} = \nabla_{u} \mathbb{E}_{\gamma}[\ell(y, y_{w,\gamma}^*)]$.
Since $\epsilon$ is a real valued number rather than a vector, we do not need to consider the directional derivative, which greatly simplifies the mathematical derivations. We define $f_n(\gamma, \hat y) \triangleq   \sum_{\alpha}\mu_{u,\alpha}(x, \hat y_\alpha)+\sigma_v(x) \sum_{i=1}^n \gamma_i (\hat y_i) + \frac{1}{n} \ell(y, \hat y)$ and follow the same derivation as above to show that 
$\frac{\partial G(u,v,0)}{\partial \epsilon} = \mathbb{E}_{\gamma}[\ell(y, y_{w,\gamma}^*)]$.
Therefore $\frac{\partial}{\partial u} \frac{\partial G(u,v,0)}{\partial \epsilon} = \nabla_{u} \mathbb{E}_{\gamma}[\ell(y, y_{w,\gamma}^*)]$.

We now turn to prove Equation (\ref{eq:dsl_sigma}). 

Since Theorem \ref{main_theorem_DSL} holds for every direction $z$:
\begin{equation}
    \frac{\partial G(u,v,\epsilon)}{\partial v} = \E_\gamma \Big[\sum_{i=1}^n \nabla_v \sigma_{v}(x) \gamma_i(y^*_i(\epsilon)) \Big].
\end{equation}
Adding a derivative with respect to $\epsilon$ we get: 
\begin{flalign}
    \frac{\partial}{\partial \epsilon} \frac{\partial G(u,v,0)}{\partial v} = \nonumber \\
    \lim_{\epsilon \rightarrow 0} &\frac{1}{\epsilon} \E_\gamma \Big[\sum_{i=1}^n \nabla_v \sigma_{v}(x) \gamma_i(y^*_i(\epsilon))- \sum_{i=1}^n \nabla_v \sigma_{v}(x) \gamma_i(y^*_i) \Big] \label{dedv}
\end{flalign}

Following the above steps, it holds that $\frac{\partial}{\partial v} \frac{\partial G(u,v,0)}{\partial \epsilon} = \nabla_{v} \mathbb{E}_{\gamma}[\ell(y, y_{w,\gamma}^*)]$. 
Equation (\ref{eq:dsl_sigma}) is attained by the Hessian symmetric entries $\frac{\partial}{\partial v} \frac{\partial G(u,v,0)}{\partial \epsilon} = \frac{\partial}{\partial \epsilon} \frac{\partial G(u,v,0)}{\partial v}$, when considering Equation (\ref{dedv}). 

\end{proof}

\section{Training and architecture details}\label{ArchitectureDetails}

Both experiments are run on NVIDIA Tesla K80 standard machine.

\subsection{Bipartite matching}

\textbf{Training} We set maximum of 2000 training epochs, and deploy early stopping with patience of 50 epochs.

Our signal embedding network $\mu$ is trained with ADAM optimizer with learning rate (lr) = 0.1 and default parameters.
The noise variance network $\sigma$ is trained with Stochastic Gradient Descent optimizer, with lr=1e-6.

\textbf{Hyper-parameters} 
We set epsilon to -12. To escape zero gradients when loss is positive, we attempt increasing epsilon by 10\%.

We learn from five noise perturbations for each permutation representation.

\textbf{Signal embedding architecture} 
The network $\mu$ has a first fully connected layer that links the sets of samples
to an intermediate representation (with 32 neurons), and a second (fully connected) layer that turns
those representations into batches of latent permutation matrices of dimension d by d each. 

\textbf{Noise variance architecture}
The network $\sigma$ has a single layer connecting input sample sequences to a single output which is then activated by a softplus activation.
We have chosen such an activation to enforce a positive sigma value.

\subsection{\label{knn_details}k-nn for image classification}
\textbf{Datasets}
We consider three benchmark datasetes: MNIST dataset of handwritten digits, Fashion-MNIST dataset of fashion apparel, and CIFAR-10 dataset of natural images (no data augmentation) with the canonical splits for training and testing.

\textbf{Training}
We train for 220 epochs.

Our signal embedding network $\mu$ is trained with ADAM optimizer, with learning rate set to 0.001 in all experiments.

The noise variance network $\sigma$ is trained with Stochastic Gradient Descent optimizer. We perform a grid search over a small number of learning rates of the noise variance network $\sigma$. 
For MNIST and Fashion-MNIST datasets $lr \in \{1e-05, 1e-06\}$, and for CIFAR-10 dataset $lr \in \{1e-06, 1e-07\}$. 

\textbf{Hyper-parameters} 
We set the number of candidate image to 800 for MNIST and Fashoin-MNIST and to 600 for CIFAR-10. Generally, our method is benefited from an increased number of candidate images due to the sparsity of the gradients resulting from the max predictors nature.
The number of query images in a batch is 100 in all experiments.

We grid search over a small number of $\epsilon$ values. For MNIST dataset $\epsilon \in \{-0.05, -0.1, -0.2\}$, for Fashion-MNIST dataset $\epsilon \in \{-0.1, -0.2\}$, for CIFAR-10 dataset $\epsilon = -0.2$. To escape zero gradients when loss is positive, we attempt increasing epsilon by 10\% up to a threshold of $-0.9999$. 

In our 'Direct Stochastic Learning' settings, we attempt a single perturbation as well as five perturbations, though in almost all cases, a single perturbation is better for k > 1, while five perturbations are better for k=1.

\textbf{Signal embedding architecture} 
For MNIST dataset the following embedding $\mu$ network is deployed:

\hspace*{10mm} Conv[Kernel: 5x5, Stride: 1, Output: 24x24x20, Activation: Relu]\\
\hspace*{10mm} Pool[Stride: 2, Output: 12x12x20]\\
\hspace*{10mm} Conv[Kernel: 5x5, Stride: 1, Output: 8x8x50, Activation: Relu]\\
\hspace*{10mm} Pool[Stride: 2, Output: 4x4x50]\\
\hspace*{10mm} FC[Units: 500, Activation: Relu]

For the Fashion-MNIST and CIFAR datasets embedding networks $\mu$, we use the ResNet18 architecture
as described in \url{https://github.com/kuangliu/pytorch-cifar}.

\textbf{Noise variance architecture}
For MNIST and Fashion-MNIST datasets, the noise learning network $\sigma$ is as follows:

\hspace*{10mm} Conv[Kernel: 5x5, Stride: 1, Output: 24x24x20, Activation: Relu]\\
\hspace*{10mm} Pool[Stride: 2, Output: 12x12x20]\\
\hspace*{10mm} Conv[Kernel: 5x5, Stride: 1, Output: 8x8x50, Activation: Relu]\\
\hspace*{10mm} Pool[Stride: 2, Output: 4x4x50]\\
\hspace*{10mm} FC[Units: 500, Activation: Relu]\\
\hspace*{10mm} FC[Units: 1, Activation: Softplus]

For CIFAR-10 dataset, the noise learning network $\sigma$ is as follows:

\hspace*{10mm} Conv[Kernel: 5x5, Stride: 1, Output: 28x28x20, Activation: Relu]\\
\hspace*{10mm} Pool[Stride: 2, Output: 14x14x20]\\
\hspace*{10mm} FC[Units: 1, Activation: Softplus]\\

We have chosen the Softplus activation to enforce a positive sigma value.

\end{document}